\documentclass[twoside, 11pt]{article}

\usepackage[preprint, abbrvbib]{jmlr2e}

\usepackage{mathtools}
\usepackage{stmaryrd}
\usepackage{enumitem}
\usepackage{booktabs}
\usepackage{graphicx}
\usepackage{mathabx}
\usepackage{float}
\usepackage{subcaption}

\usepackage[ruled, vlined]{algorithm2e}
\newtheorem{assumption}[theorem]{Assumption}

\firstpageno{1}
\begin{document}

\title{Contrast-Free ICA and Causal Inference via Wasserstein Distances to the Gaussian}

\author{\name F\'elix Laplante\thanks{Corresponding author.} \email felixlaplante.research@gmail.com \\
\addr LaMME, Universit\'e Paris-Saclay \\
\AND \name Christophe Ambroise \email christophe.ambroise@univ-evry.fr \\
\addr LaMME, Universit\'e Paris-Saclay \\
\AND \name Pierre Humbert \email pierre.humbert@cnrs.fr \\
\addr LaMME, Universit\'e Paris-Saclay}

\maketitle

\begin{abstract}
We study the squared $2$-Wasserstein distance to the standard Gaussian as a non-Gaussianity criterion and use it for linear Independent Component Analysis (ICA) and causal discovery in Linear Non-Gaussian Acyclic Models (LiNGAM). Unlike commonly used parametric contrasts and approximations of information-theoretic quantities, this criterion requires no distributional regularity beyond finite second moments, involves neither approximation nor tuning parameters, and can be computed exactly and efficiently from empirical order statistics. Our analysis relies on a strict subadditivity property of the $2$-Wasserstein distance to the Gaussian. At the population level, we prove exact identification of the ICA unmixing matrix, up to signed permutation, and give an analogous characterization of causal orders through sequential least-squares residuals. We then define empirical plug-in estimators and prove distribution-free uniform convergence under finite-moment assumptions, before detailing three practical solvers: a Picard-style orthogonal optimizer for ICA, an exhaustive dynamic program for causal order search, and a greedy order search variant. Empirically, we demonstrate competitive performance for both tasks and provide open-source implementations for source separation and causal discovery.
\end{abstract}


\section*{Introduction}

Linear independent component analysis for source separation \citep{jutten1991blind} and linear non-Gaussian acyclic models for causal inference \citep{shimizu2006linear} are two closely related, well-established problems that rely on independence and non-Gaussianity to ensure identifiability. In ICA, one observes an invertible linear mixture of independent latent sources and seeks to recover them up to permutation and scaling. In linear causal inference, one observes variables generated by independent structural noises through an acyclic linear system and seeks to recover the causal order and eventually the underlying directed graph.

Many classical ICA algorithms seek an unmixing matrix such that the recovered sources are as non-Gaussian as possible, based on higher-order moments, cumulants, likelihood surrogates, or approximations of negentropy \citep{comon1994independent, bell1995information, hyvarinen1999fast, hyvarinen2000independent}. While often effective in practice, these parametric methods generally do not exploit all the information in the source distributions and may lead to incorrect recovery. Furthermore, many require choosing a contrast or nonlinearity, whose suitability depends on the source distributions, often assumed to be absolutely continuous with respect to Lebesgue measure.

These limitations are also relevant to causal discovery. Indeed, without additional restrictions, a linear Gaussian acyclic model does not uniquely determine the underlying causal graph, whereas linear non-Gaussian models enjoy the same identifiability principle as ICA. Early methods such as ICA-LiNGAM \citep{shimizu2006linear} exploited this connection by applying ICA to recover the structural model. More recent approaches, such as DirectLiNGAM \citep{shimizu2011directlingam}, avoid using ICA by recursively finding exogenous variables using nonparametric independence criteria \citep{gretton2005measuring}.

Overall, whether in ICA or causal discovery, identification depends on the non-Gaussianity of the source distributions.

\bigbreak

This paper studies the squared $2$-Wasserstein distance of a standardized random variable to the standard Gaussian as a nonparametric non-Gaussianity measure. The Wasserstein route provides an exact, easily computable, one-dimensional optimal-transport quantity that admits a direct order-statistics plug-in. We call it \emph{contrast-free} in a specific sense: although the criterion is itself a contrast, it is fixed by the Gaussian reference and has no tunable nonlinearity or moment order. This makes the resulting ICA and causal order estimators distribution-free. Apart from standardization, non-Gaussianity, and finite moments for concentration, the theory does not impose a parametric source family, density smoothness, cumulant condition, or likelihood model.

\paragraph{Main Contributions.}

Our contributions are as follows. Under the usual linear ICA and LiNGAM assumption that at most one source is Gaussian, we first prove that a nontrivial orthogonal mixture of independent standardized sources has strictly smaller total Wasserstein non-Gaussianity than the unmixed sources (Lemmas~\ref{lem:wasserstein-linear-combination} and~\ref{lem:strict-wasserstein}). We then derive an optimal transport-based objective whose oracle maximizers are exactly the true unmixing matrices, up to signed permutations (Theorem~\ref{thm:ica-identifiability}), and extend the methodology to identify causal orders in linear non-Gaussian Structural Equation Models (SEMs) (Theorem~\ref{thm:causal-order-identification}). We further establish distribution-free uniform convergence and consistency of global empirical maximizers for ICA unmixing (Theorem~\ref{thm:uniform-convergence} and Corollary~\ref{cor:ica-estimator-consistency}) as well as for causal order estimation (Theorems~\ref{thm:dag-uniform-convergence}, \ref{thm:dag-order-sub-gaussian-rate}, and Corollary~\ref{cor:dag-order-consistency}). Finally, we develop a quasi-Newton solver for ICA as well as exact and greedy order search procedures for causal inference with exactness guarantees (Theorem~\ref{thm:greedy-equal-score}), and evaluate their recovery and computational cost in reproducible, controlled experiments (Sections~\ref{sec:algorithms} and~\ref{sec:experiments}).

\section{Related Work}

ICA is a classical model for blind source separation. Its standard identifiability theory shows that independent non-Gaussian sources can be recovered from linear mixtures up to signed permutations, provided that at most one source is Gaussian \citep{comon1994independent}. Early and influential algorithms include cumulant-based and information-theoretic methods such as JADE, Infomax, and FastICA \citep{cardoso1993blind, bell1995information, hyvarinen1999fast, hyvarinen2000independent}. Modern ICA solvers also exploit second-order optimization on the whitened orthogonal problem. For instance, the Picard algorithm uses a preconditioned quasi-Newton scheme and is particularly relevant to our Wasserstein implementation \citep{ablin2018faster}. Nonparametric ICA methods optimize distributional dependence criteria rather than fixed low-order cumulants or prescribed nonlinearities. For instance, quasi-nonparametric neural estimators such as MINE learn a variational estimate of mutual information \citep{belghazi2018mine}. Kernel ICA instead uses reproducing-kernel dependence measures \citep{bach2002kernel}, while distance-covariance ICA of \cite{matteson2017independent} seeks a rotation minimizing a distance-covariance measure of mutual dependence after whitening. Since distance covariance vanishes if and only if the components are independent, the criterion is necessary and sufficient for independence, and consistency is established without assuming the existence of independent components a priori. Other recent optimal-transport approaches use multivariate ranks and related rank-based independence criteria \citep{niu2022distribution}. Particularly close is the Wasserstein projection-pursuit method of \cite{mukherjee2023wasserstein}, which also maximizes one-dimensional Wasserstein distances to the Gaussian. However, it targets subspace recovery in a spiked model, whereas our ICA result recovers the individual source directions, up to sign and permutation.

Linear non-Gaussian causal inference techniques build on the same source-separation principle. For instance, \cite{shimizu2006linear} shows that linear acyclic structural equation models with independent non-Gaussian noises are identifiable from observational data. Building on this insight, DirectLiNGAM provides a greedy sequential procedure that avoids an ICA preprocessing step \citep{shimizu2011directlingam}. Score-based structure learning instead optimizes decomposable local scores over Directed Acyclic Graphs (DAGs). The dynamic program of \cite{silander2006simple} and the A\textsuperscript{*}-based algorithm of \cite{yuan2013learning} give exact combinatorial solvers for globally optimal Bayesian network structures with exponential computational cost. For linear SEMs, identifiable score-based approaches have been obtained under additional restrictions, including known or equal variances \citep{loh2014high, chen2019causal, park2020identifiability, peters2014identifiability}. Finally, continuous relaxations such as NOTEARS \citep{zheng2018dags} and DAGMA \citep{bello2022dagma} replace the combinatorial acyclicity constraint by smooth characterizations. They are scalable, but their solutions depend on nonconvex continuous optimization and regularization choices.

\bigbreak

In the interest of full transparency, the first version of the present work was submitted to arXiv, and the following day, the authors of \cite{jha2026linear} independently submitted their closely related preprint. They also use the squared $2$-Wasserstein distance to the standard Gaussian as an ICA non-Gaussianity criterion, but only establish population-level recovery. Indeed, their theoretical analysis is restricted to the oracle ICA setting under strong smoothness conditions, whereas the present work additionally establishes empirical uniform convergence and estimator consistency, extends the criterion to causal order identification and estimation, and develops exact and greedy causal order solvers.

\section{Optimal Transport Preliminaries} \label{sec:ot-preliminaries}

Optimal Transport (OT) provides the fundamental metric used throughout the paper. More specifically, we use the $2$-Wasserstein distance. We refer the reader to \cite{villani2009optimal}, \cite{peyre2019computational}, and \cite{brenier1991polar} for general references and foundational results on optimal transport. We now recall the optimal-transport notions used below, beginning with the set of distributions with finite moments on which Wasserstein distances are well-defined.

\begin{definition}[$p$-Moment Probability Measures]
Let $p \geq 1$. We denote by $\mathcal{P}_p(\mathbb{R}^d)$ the set of Borel probability measures on $\mathbb{R}^d$ with finite $p$-th moment:
\[
\mathcal{P}_p(\mathbb{R}^d) \coloneqq \bigl\{ \mu \in \mathcal{P}(\mathbb{R}^d) : \int_{\mathbb{R}^d} \Vert x \Vert^p \, d\mu(x) < \infty \bigr\}.
\]
\end{definition}

Although the definition of the Wasserstein distance is stated for $p \geq 1$, we mainly use the quadratic case $p = 2$. Throughout the following sections, we identify random variables with their distributions, so we may write Wasserstein distances in terms of random variables instead.

\begin{definition}[$p$-Wasserstein Distance]
Let $\mu, \nu \in \mathcal{P}_p(\mathbb{R}^d)$ and let $\Pi(\mu, \nu)$ denote the set of probability measures on $\mathbb{R}^d \times \mathbb{R}^d$ with marginals $\mu$ and $\nu$. The $p$-Wasserstein distance between $\mu$ and $\nu$ is defined by
\[
\mathcal{W}_p(\mu, \nu)^p \coloneqq \inf_{\pi \in \Pi(\mu, \nu)} \int_{\mathbb{R}^d \times \mathbb{R}^d} \Vert x - y \Vert^p \, d\pi(x, y) = \inf_{X \sim \mu, Y \sim \nu} \mathbb{E}\bigl[ \Vert X - Y \Vert^p \bigr].
\]
\end{definition}

For the projected distance used in the concentration analysis, we first define pushforward measures.

\begin{definition}[Pushforward Measure] \label{def:pushforward-measure}
Let $\mu$ be a Borel probability measure on $\mathbb{R}^d$ and let $T$ be a measurable map from $\mathbb{R}^d$ to $\mathbb{R}^k$. The pushforward of $\mu$ by $T$, denoted $T_\sharp\mu$, is the Borel probability measure on $\mathbb{R}^k$ defined by
\[
T_\sharp\mu(A) \coloneqq \mu\bigl( T^{-1}(A) \bigr),
\]
for every Borel set $A \subseteq \mathbb{R}^k$. Equivalently, if $X \sim \mu$, then $T(X) \sim T_\sharp\mu$.
\end{definition}

The max-sliced distance is then defined as the largest one-dimensional Wasserstein distance over all directions. This quantity will be used in the proof of Theorem~\ref{thm:uniform-convergence} in Appendix~\ref{pf:ica-uniform-convergence} to uniformly control the projected empirical Wasserstein errors over the orthogonal ICA parameter space.

\begin{definition}[Max-Sliced Wasserstein Distance] \label{def:max-sliced-wasserstein}
Let $\mu, \nu \in \mathcal{P}_p(\mathbb{R}^d)$. For $\theta \in \mathbb{S}^{d - 1} \coloneqq \bigl\{ x \in \mathbb{R}^d : \Vert x \Vert = 1 \bigr\}$, let $\pi_{\theta}$ be the one-dimensional projection from $\mathbb{R}^d$ to $\mathbb{R}$ defined by $\pi_{\theta}: x \mapsto \theta^\top x$. We denote the max-sliced $p$-Wasserstein distance by
\[
\widebar{\mathcal{W}}_p(\mu, \nu) \coloneqq \sup_{\theta \in \mathbb{S}^{d - 1}} \mathcal{W}_p\bigl( (\pi_{\theta})_\sharp\mu, (\pi_{\theta})_\sharp\nu \bigr).
\]
\end{definition}

Since the proof of the strict Wasserstein comparison in Appendix~\ref{pf:strict-wasserstein} uses the existence and uniqueness of quadratic optimal transport maps, we also recall Brenier's theorem in the form needed later. We refer the reader to the proof in \cite{brenier1991polar}.

\begin{theorem}[Brenier's Theorem] \label{thm:brenier}
Let $\mu, \nu \in \mathcal{P}_2(\mathbb{R}^d)$, and assume that $\mu$ is absolutely continuous with respect to Lebesgue measure. Then there exists a convex function $\varphi : \mathbb{R}^d \to \mathbb{R}$ such that
\[
T \coloneqq \nabla \varphi, \quad \nu = T_\sharp \mu, \quad \mathcal{W}_2(\mu, \nu)^2 = \mathbb{E}_{X \sim \mu}\Bigl[ \bigl\Vert X - T(X) \bigr\Vert^2 \Bigr].
\]

Furthermore, the optimal coupling $\pi = (\mu, T_\sharp \mu)$ is unique and induced by the measurable map $T$.
\end{theorem}

\section{Optimal Transport ICA} \label{sec:ot-ica}

In this section, we consider linear independent component analysis and assume that the data distribution is centered and whitened. This population-level normalization is standard in ICA analyses \citep{comon1994independent, hyvarinen2001independent} and reduces a general invertible ICA model to an orthogonal one, leaving only sign and permutation indeterminacies. Empirically, a consistent plug-in strategy consists of subtracting the sample mean and applying a whitening map, for instance through the eigendecomposition of the empirical covariance matrix. On the resulting orthogonal parameter space, our estimation strategy is to maximize the squared $2$-Wasserstein distance of each recovered coordinate to the Gaussian. Since this quantity is computed exactly in one dimension, the resulting empirical objective does not require density estimation, entropy approximation, or the choice of a nonlinearity.

\subsection{Non-Gaussian Orthogonal Model}

Let $d \geq 2$ denote the number of variables, let $X = (X_1, \dots, X_d)^\top$ be a random vector with distribution $P$, and let $S = (S_1, \dots, S_d)^\top$ denote the latent sources. We assume that the data follow the model
\begin{equation} \label{eq:ica-model}
X = (W^\star)^{-1} S,
\end{equation}
where $(W^\star)^{-1} \in \mathcal{O}_d(\mathbb{R})$ is the \emph{mixing} matrix, $W^\star$ is the \emph{unmixing} matrix, and $\mathcal{O}_d(\mathbb{R})$ denotes the space of real-valued orthogonal matrices of size $d \times d$.

\bigbreak

The following assumptions are standard in ICA. The first normalizes the latent sources, while the second ensures identifiability up to signed permutations.

\begin{assumption}[Source Standardization] \label{ass:ica-standardization}
The components of $S$ are mutually independent and satisfy
\[
\forall j \in \llbracket d \rrbracket, \, \mathbb{E}[S_j] = 0, \quad \mathbb{E}[S_j^2] = 1.
\]
\end{assumption}

\begin{assumption}[ICA Non-Gaussianity] \label{ass:ica-non-gaussianity}
At most one component of $S$ is Gaussian.
\end{assumption}

In particular, Assumption~\ref{ass:ica-standardization} together with the orthogonality of the unmixing matrix $W^\star$ imply that the distribution of $X$ is centered and whitened
\[
\mathbb{E}[X] = (W^\star)^{-1} \mathbb{E}[S] = 0, \quad \mathrm{Cov}(X) = (W^\star)^{-1} \mathrm{Cov}(S) W^\star = I_d.
\]

\subsection{Optimal Transport-Based Identifiability}

We now focus on the Wasserstein inequalities that drive our identifiability guarantees and study one-dimensional projections of the distribution $P$ along a unit vector $\alpha \in \mathbb{R}^d$. The starting point is a non-strict subadditivity bound for independent centered variables. This bound will then be applied with independent standard Gaussian variables as the reference components, so that any unit-norm Gaussian linear combination remains standard Gaussian. This particular case has already been studied, for instance, in \cite{panaretos2019statistical}.

\begin{lemma}[Wasserstein Bound for Linear Combinations] \label{lem:wasserstein-linear-combination}
Let $S_1, \dots, S_d$ and $Z_1, \dots, Z_d$ be independent, centered random variables with finite second moments. Then for any $\alpha_1, \dots, \alpha_d \in \mathbb{R}$
\[
\mathcal{W}_2\bigl( \sum_{j = 1}^d \alpha_j S_j, \sum_{j = 1}^d \alpha_j Z_j \bigr)^2 \leq \sum_{j = 1}^d \alpha_j^2 \mathcal{W}_2\bigl( S_j, Z_j \bigr)^2.
\]
\end{lemma}

\bigbreak

The proof of Lemma~\ref{lem:wasserstein-linear-combination} is provided in Appendix~\ref{pf:wasserstein-linear-combination}.

\bigbreak

We now define the population objective used to identify the independent components.

\begin{definition}[Oracle OT-ICA Objective] \label{def:ica-objective-function}
For every candidate orthogonal \emph{unmixing} matrix $W \in \mathcal{O}_d(\mathbb{R})$, define
\[
F(W) \coloneqq \sum_{j = 1}^d \mathcal{W}_2\bigl( (WX)_j, \mathcal{N}(0, 1) \bigr)^2.
\]
\end{definition}

The objective measures the total Wasserstein non-Gaussianity of the recovered components. For identifiability, it must distinguish projections aligned with individual sources from projections that combine several sources. A non-strict Wasserstein bound is therefore insufficient, since it would not exclude a nontrivial mixture from attaining the same objective value as a source-aligned projection.

The following lemma ensures that every normalized, nontrivial mixture is strictly closer to the standard Gaussian than the corresponding weighted combination of the individual source distances. The two-variable case was established in Proposition~2.4 of \cite{johnson2005central}, but we state the general form and provide a proof for completeness.

\begin{lemma}[Strict Wasserstein Bound for Nontrivial Mixtures] \label{lem:strict-wasserstein}
Suppose Assumptions~\ref{ass:ica-standardization} and~\ref{ass:ica-non-gaussianity} hold. Then for every $\alpha \in \mathbb{R}^d$ such that $\Vert \alpha \Vert_2 = 1$ and $\#\bigl\{ j \in \llbracket d \rrbracket : \alpha_j \neq 0 \bigr\} > 1$, we have
\[
\mathcal{W}_2\bigl( \sum_{j = 1}^d \alpha_j S_j, \mathcal{N}(0, 1) \bigr)^2 < \sum_{j = 1}^d \alpha_j^2 \mathcal{W}_2\bigl( S_j, \mathcal{N}(0, 1) \bigr)^2.
\]
\end{lemma}

\bigbreak

The proof of Lemma~\ref{lem:strict-wasserstein} is provided in Appendix~\ref{pf:strict-wasserstein}.

\bigbreak

This strictness now provides the separation needed for the Wasserstein objective below.

\begin{theorem}[Oracle Identifiability of the OT-ICA Objective] \label{thm:ica-identifiability}
Suppose Assumptions~\ref{ass:ica-standardization} and~\ref{ass:ica-non-gaussianity} hold. Then
\[
\operatorname*{\arg\max}_{W \in \mathcal{O}_d(\mathbb{R})} F(W) = \{ MW^\star : M \in \mathcal{M}_d \},
\]
where $\mathcal{M}_d \coloneqq \bigl\{ \mathrm{diag}(\epsilon) \Pi : \epsilon \in \{ -1, 1 \}^d, \Pi \text{ permutation matrix} \bigr\}$ is the set of signed permutation matrices of size $d \times d$.
\end{theorem}

\bigbreak

The proof of Theorem~\ref{thm:ica-identifiability} is provided in Appendix~\ref{pf:ica-identifiability}.

\bigbreak

In other words, maximizing the oracle objective $F$ recovers the unmixing matrix $W^\star$ up to signed permutation because every rotation that mixes multiple sources scores strictly lower. Furthermore, the criterion is nonparametric and distribution-free, fixed by the Gaussian reference, and requires only finite second moments under the stated assumptions.

\subsection{Empirical Estimation}

We now replace the unknown true distribution in the objective (Definition~\ref{def:ica-objective-function}) by its empirical counterpart. In the following, $n \geq 1$ denotes the sample size, and we observe samples $X^{(1)}, \dots, X^{(n)} \overset{\mathrm{i.i.d.}}{\sim} P$. We further define the empirical distribution
\[
P_n \coloneqq \frac{1}{n} \sum_{i = 1}^n \delta_{X^{(i)}}.
\]

The estimation scheme is then to maximize the empirical objective
\[
\widehat{W}_n \in \operatorname*{\arg\max}_{W \in \mathcal{O}_d(\mathbb{R})} \widehat{F}_n(W),
\]
where
\begin{equation}
\widehat{F}_n(W) \coloneqq \sum_{j = 1}^d \mathcal{W}_2\bigl( (\pi_{W_j})_\sharp P_n, \mathcal{N}(0, 1) \bigr)^2. \label{eq:empirical-ica-objective}
\end{equation}

\begin{remark}
Each $2$-Wasserstein distance to the Gaussian can be computed in $O(n \log n)$ time (see Lemma~\ref{lem:empirical-gaussian-wasserstein}).
\end{remark}

\begin{remark}
Since $\widehat{F}_n$ is continuous on the compact set $\mathcal{O}_d(\mathbb{R})$, the maximum in Equation~\eqref{eq:empirical-ica-objective} is attained.
\end{remark}

\subsection{Convergence Guarantees}

We first establish convergence of the empirical objective, then use the oracle identifiability to derive consistency of its global maximizers. The result is distribution-free in the sense that the bound is not tied to a parametric source model, a density estimator, or a selected contrast nonlinearity. Following the max-sliced Wasserstein concentration result of \cite{han2024maxsliced}, we assume a finite moment of order strictly larger than four. This condition is verified, for instance, when $X$ is bounded, which also entails the non-Gaussianity of the sources.

\begin{assumption}[Moment Condition for Empirical Convergence] \label{ass:moment}
There exists $p > 4$ such that
\[
P \in \mathcal{P}_p(\mathbb{R}^d).
\]
\end{assumption}

\begin{theorem}[Uniform Convergence of the Empirical Objective] \label{thm:uniform-convergence}
Suppose $n \geq d \log(2n + 1)^{1 + 4/p}$ and Assumptions~\ref{ass:ica-standardization} and~\ref{ass:moment} hold. Then there exists a constant $C_p(P) > 0$, depending only on $p$ and $\mathbb{E}[ \Vert X \Vert^p ]$, such that
\[
\mathbb{E}\Bigl[ \sup_{W \in \mathcal{O}_d(\mathbb{R})} \bigl\vert \widehat{F}_n(W) - F(W) \bigr\vert \Bigr] \leq C_p(P) d^{5/4} n^{-1/4} \log(2n + 1)^{1/p + 1/4}.
\]
\end{theorem}

\begin{remark}[$n^{-1/4}$ Convergence Rate]
Up to logarithmic terms, the achieved convergence rate matches the $n^{-1/4}$ rate established in \cite{mukherjee2023wasserstein}.
\end{remark}

\bigbreak

The proof of Theorem~\ref{thm:uniform-convergence} is provided in Appendix~\ref{pf:ica-uniform-convergence}.

\bigbreak

Theorem~\ref{thm:uniform-convergence} shows that the empirical objective uniformly approximates the oracle objective over $\mathcal{O}_d(\mathbb{R})$. To translate this objective convergence into recovery of an unmixing matrix, we must account for the equivalence class described by Theorem~\ref{thm:ica-identifiability}: the true unmixing matrix is recoverable only up to a signed permutation. We therefore choose to measure recovery by the Frobenius distance between the relative unmixing $W (W^\star)^{-1}$ and the finite set of signed permutation matrices.

\begin{definition}[Signed Permutation Recovery Error] \label{def:ica-recovery-error}
For every $W \in \mathcal{O}_d(\mathbb{R})$, define its signed permutation recovery error by
\[
\rho(W) \coloneqq \inf_{M \in \mathcal{M}_d} \Vert W (W^\star)^{-1} - M \Vert_F,
\]
where $\mathcal{M}_d$ is defined in Theorem~\ref{thm:ica-identifiability}.
\end{definition}

The quantity $\rho(W)$ vanishes if and only if $W$ recovers the source directions up to signed permutation and otherwise measures their Frobenius discrepancy. It is also used to quantify the recovery error in \cite{niu2022distribution}, but the metric remains largely arbitrary. For a fixed tolerance $\delta > 0$, the next quantity measures how much oracle objective value in $F$ is lost by every unmixing matrix whose recovery error is at least $\delta$.

\begin{definition}[Oracle Separation Margin] \label{def:ica-separation-margin}
For every $\delta > 0$, define the margin
\[
\Gamma(\delta) \coloneqq F(W^\star) - \max_{W \in \mathcal{K}_\delta} F(W),
\]
where $\mathcal{K}_\delta \coloneqq \bigl\{ W \in \mathcal{O}_d(\mathbb{R}) : \rho(W) \geq \delta \bigr\}$ is assumed to be nonempty, which always holds for $\delta$ small enough.
\end{definition}

By continuity of $F$, compactness of $\mathcal{K}_\delta$, and Theorem~\ref{thm:ica-identifiability} with its required assumptions, we have $\Gamma(\delta) > 0$ whenever $\mathcal{K}_\delta$ is nonempty. Uniform convergence can therefore be converted into recovery of the unmixing matrix.

\begin{corollary}[Consistency of Empirical Unmixing Matrices] \label{cor:ica-estimator-consistency}
Suppose $n \geq d \log(2n + 1)^{1 + 4/p}$ and Assumptions~\ref{ass:ica-standardization}, \ref{ass:ica-non-gaussianity} and~\ref{ass:moment} hold. If $\delta > 0$ and $\mathcal{K}_\delta$ is nonempty, then there exists a constant $C_p(P) > 0$, depending only on $p$ and $\mathbb{E}[ \Vert X \Vert^p ]$, such that
\[
\mathbb{P}\bigl( \rho(\widehat{W}_n) \geq \delta \bigr) \leq \mathbb{P}\Bigl( \sup_{W \in \mathcal{O}_d(\mathbb{R})} \bigl\vert \widehat{F}_n(W) - F(W) \bigr\vert \geq \frac{\Gamma(\delta)}{2} \Bigr) \leq \frac{C_p(P) d^{5/4} n^{-1/4} \log(2n + 1)^{1/p + 1/4}}{\Gamma(\delta)}.
\]

Thus, we have convergence in probability $\rho(\widehat{W}_n) \overset{\mathbb{P}}{\to} 0$ and in expectation $\mathbb{E}\bigl[ \rho(\widehat{W}_n) \bigr] \to 0$.
\end{corollary}

\bigbreak

The proof of Corollary~\ref{cor:ica-estimator-consistency} is provided in Appendix~\ref{pf:ica-estimator-consistency}.

\bigbreak

Corollary~\ref{cor:ica-estimator-consistency} establishes consistency of global empirical maximizers without an additional margin assumption. The probability bound is quantitative at every fixed recovery tolerance $\delta$. A direct rate for $\rho(\widehat{W}_n)$ would require a lower bound on $\Gamma(\delta)$ as $\delta \to 0$. The result also does not apply directly to a local numerical solution unless its optimization error is controlled.

\section{Causal Inference in Linear Non-Gaussian SEMs}

In this section, we further specialize our framework and show that the Wasserstein strict subadditivity principle can also be used for causal order recovery in linear non-Gaussian acyclic models. As in many two-stage approaches, including ICA-LiNGAM \citep{shimizu2006linear}, we first focus on \emph{order} recovery. Given a causal order, inferring the structural coefficients then reduces to a sequence of sparse regression problems, which can be solved efficiently by adaptive Lasso-type procedures \citep{zou2006adaptive, shimizu2006linear}. As in the ICA analysis, the data distribution is assumed to be centered at the population level, which empirically corresponds to a centering step obtained by subtracting the sample mean. Intuitively, in the oracle setting, the correct causal order yields sequential least-squares residuals that coincide with structural noises. In a wrong order, at least one residual combines several independent noises, bringing it closer to a Gaussian distribution. Hence, the strict Wasserstein inequality can distinguish correct orders from incorrect ones.

More precisely, orthogonal rotations in the ICA analysis and sequential least-squares residualization under a candidate causal order both produce one-dimensional linear combinations of the independent components. Lemma~\ref{lem:strict-wasserstein} therefore applies to the recovered ICA coordinates and to the standardized residuals.

\subsection{Non-Gaussian Model}

Let $d \geq 2$ denote the number of variables, let $X = (X_1, \dots, X_d)^\top$ be a random vector with distribution $P$, and let $\varepsilon = (\varepsilon_1, \dots, \varepsilon_d)^\top$ denote the structural noises. We assume that the data follow the model
\begin{equation} \label{eq:lingam}
X = B^\star X + \varepsilon,
\end{equation}
where $B^\star \in \mathbb{R}^{d \times d}$ is the structural matrix, whose associated graph is assumed to be \emph{acyclic}, meaning that it is impossible to start from one variable, follow a sequence of directed relationships, and eventually return to the same variable.

\bigbreak

Formally, to each structural matrix $B^\star$, one associates a directed graph $\mathcal{G}^\star$ with vertices $\llbracket d \rrbracket$ and edges $E^\star \coloneqq \bigl\{ (j, k) \in \llbracket d \rrbracket^2 : B^\star_{kj} \neq 0 \bigr\}$. Thus, Model~\eqref{eq:lingam} is equivalently written coordinatewise as an SEM
\[
\forall j \in \llbracket d \rrbracket, \, X_j = \sum_{k \in \mathrm{Pa}_{\mathcal{G}^\star}(j)} B^\star_{jk} X_k + \varepsilon_j, \quad \text{where } \mathrm{Pa}_{\mathcal{G}^\star}(j) \coloneqq \bigl\{ k \in \llbracket d \rrbracket : B^\star_{jk} \neq 0 \bigr\},
\]
where $\mathrm{Pa}_{\mathcal{G}^\star}(j)$ denotes the set of parents of node $j \in \llbracket d \rrbracket$, that is, the variables having a direct effect on it. Furthermore, because $\mathcal{G}^\star$ is acyclic, there exists a \emph{causal ordering} (otherwise known as a topological ordering) of its vertices under which the corresponding permuted structural matrix is strictly lower triangular.

\begin{definition}[Causal Orders] \label{def:topological-order}
Consider an order $\sigma \in \mathfrak{S}_d$, $\mathfrak{S}_d$ being the set of permutations of $\llbracket d \rrbracket$, and denote by $\Pi^\sigma \in \mathbb{R}^{d \times d}$ its associated permutation matrix. The set of causal orders associated with Model~\eqref{eq:lingam} is defined by
\[
\mathcal{I}^\star \coloneqq \bigl\{ \sigma \in \mathfrak{S}_d : \Pi^\sigma B^\star (\Pi^\sigma)^\top \text{ is strictly lower triangular} \bigr\}.
\]
\end{definition}

In other words, we represent an order $\sigma \in \mathfrak{S}_d$ by the tuple $\bigl( \sigma(1), \dots, \sigma(d) \bigr)$, which lists variables from source (no parents) to sink (no children). We further define the \emph{predecessors} of $j \in \llbracket d \rrbracket$ as the preceding variables in this order, i.e., the set $A_j(\sigma) \coloneqq \bigl\{ k \in \llbracket d \rrbracket : \sigma^{-1}(k) < \sigma^{-1}(j) \bigr\}$. An order $\sigma^\star \in \mathcal{I}^\star$ is causal precisely when every parent appears among the predecessors of its child.

We now impose the standard LiNGAM requirements of noise normalization and non-Gaussianity.

\begin{assumption}[Structural Noise Normalization] \label{ass:causal-noise-normalization}
The structural noises $\varepsilon_1, \dots, \varepsilon_d$ are mutually independent, and
\[
\forall j \in \llbracket d \rrbracket, \, \mathbb{E}[\varepsilon_j] = 0, \quad \mathbb{E}[\varepsilon_j^2] \in (0, +\infty).
\]
\end{assumption}

\begin{assumption}[Structural Noise Non-Gaussianity] \label{ass:causal-non-gaussianity}
At most one of the structural noises $\varepsilon_1, \dots, \varepsilon_d$ is Gaussian.
\end{assumption}

\subsection{Optimal Transport-Based Identifiability}

The representation $X = (I_d - B^\star)^{-1} \varepsilon$ expresses the observations as linear mixtures of the independent structural noises. One may therefore estimate $I_d - B^\star$ through ICA, which guarantees identifiability. This is the approach of ICA-LiNGAM, which first estimates $B^\star$ by ICA and then extracts a causal order from the resulting estimate, as detailed in Paragraph~\ref{subsec:ica-lingam}. Instead, we propose to optimize our Wasserstein criterion directly over causal orders through sequential residualization, which follows the general principle underlying DirectLiNGAM, but replaces independence testing criteria by a non-Gaussianity measure.

We first define the standardization operator used in the Wasserstein objective below.

\begin{definition}[Standardization] \label{def:standardization-causal-order}
Let $m \geq 1$ and $Z \coloneqq (Z_1, \ldots, Z_m)^\top$ be a random vector whose components have finite nonzero second moments. We define its componentwise standardization by
\[
\mathrm{std}(Z) \coloneqq \biggl( \frac{Z_1}{\sqrt{\mathbb{E}[Z_1^2]}}, \ldots, \frac{Z_m}{\sqrt{\mathbb{E}[Z_m^2]}} \biggr)^\top.
\]
\end{definition}

\begin{remark}
For a probability measure $\mu \in \mathcal{P}_2(\mathbb{R})$ with nonzero second moment, we extend the operator of Definition~\ref{def:standardization-causal-order} and let $\mathrm{std}(\mu)$ denote its pushforward by $x \mapsto x / \sqrt{\int x^2 \, d\mu(x)}$.
\end{remark}

We now introduce the counterpart of the oracle ICA objective, which is defined over candidate orders instead of orthogonal matrices. For each order, every variable is regressed on its predecessors, and the resulting residuals are compared to the standard Gaussian.

\begin{definition}[Oracle Wasserstein Order Objective] \label{def:wasserstein-lingam-objective}
For every candidate order $\sigma \in \mathfrak{S}_d$, define
\[
G(\sigma) \coloneqq \sum_{j = 1}^d \mathcal{W}_2\Bigl( \mathrm{std}\bigl( R_j(\sigma) \bigr), \mathcal{N}(0, 1) \Bigr)^2,
\]
where, for any $j \in \llbracket d \rrbracket$, $R_j$ denotes the residual of $X_j$ after regression on its predecessors
\[
R_j(\sigma) \coloneqq X_j - \mathrm{proj}_{\mathrm{span}\bigl\{ X_k : k \in A_j(\sigma) \bigr\}} X_j,
\]
where
\[
A_j(\sigma) \coloneqq \bigl\{ k : \sigma^{-1}(k) < \sigma^{-1}(j) \bigr\}.
\]
\end{definition}

To relate the order-based objective to the ICA analysis, we first show that the standardized residuals corresponding to any candidate order form an orthogonal mixture of the standardized structural noises.

\begin{lemma}[Orthogonal Representation of Order Residuals] \label{lem:order-residuals-orthogonal}
Suppose Assumption~\ref{ass:causal-noise-normalization} holds. Then for every order $\sigma \in \mathfrak{S}_d$
\[
\exists Q^\sigma \in \mathcal{O}_d(\mathbb{R}), \, \mathrm{std}\bigl( R(\sigma) \bigr) = Q^\sigma \mathrm{std}(\varepsilon).
\]
\end{lemma}

\bigbreak

The proof of Lemma~\ref{lem:order-residuals-orthogonal} is provided in Appendix~\ref{pf:order-residuals-orthogonal}.

\bigbreak

Replacing $\mathrm{std}\bigl( R_j(\sigma) \bigr)$ by $(Q^\sigma \mathrm{std}(\varepsilon))_j$ in Definition~\ref{def:wasserstein-lingam-objective}, we recognize our proposed ICA criterion of Definition~\ref{def:ica-objective-function}. Thus, every candidate order induces the same type of orthogonal mixing considered in the ICA analysis. The equality case additionally requires the following causal characterization: an order whose residuals recover individual structural noises without mixing must be causal. This guarantee is established in the following proposition.

\begin{proposition}[Signed Permutation Implies a Causal Order] \label{prop:permutation-causal}
Suppose Assumption~\ref{ass:causal-noise-normalization} holds. If, for some order $\sigma \in \mathfrak{S}_d$, the matrix $Q^\sigma$ from Lemma~\ref{lem:order-residuals-orthogonal} is a signed permutation matrix, then $\sigma \in \mathcal{I}^\star$.
\end{proposition}

\bigbreak

The proof of Proposition~\ref{prop:permutation-causal} is provided in Appendix~\ref{pf:permutation-causal}.

\bigbreak

Combining the orthogonal residual representation of Lemma~\ref{lem:order-residuals-orthogonal} and the causal order characterization of Proposition~\ref{prop:permutation-causal}, the following theorem now yields causal order identification.

\begin{theorem}[Identification of Causal Orders] \label{thm:causal-order-identification}
Suppose Assumptions~\ref{ass:causal-noise-normalization} and~\ref{ass:causal-non-gaussianity} hold. Then
\[
\operatorname*{\arg\max}_{\sigma \in \mathfrak{S}_d} G(\sigma) = \mathcal{I}^\star.
\]
\end{theorem}

\bigbreak

The proof of Theorem~\ref{thm:causal-order-identification} is provided in Appendix~\ref{pf:causal-order-identification}.

\bigbreak

In other words, Theorem~\ref{thm:causal-order-identification} shows that the oracle objective identifies exactly the set of causal orders, rather than a unique permutation, since a directed acyclic graph may admit several causal orderings. However, note that the result only concerns order recovery. Once a causal order is known, the structural coefficients still have to be (sparsely) estimated by regressing each variable on its predecessors.

\subsection{Empirical Estimation} \label{subsec:lingam-estim}

Let $n \geq 1$ and $X^{(1)}, \dots, X^{(n)} \overset{\mathrm{i.i.d.}}{\sim} P$ be the observed sample. For every candidate order $\sigma \in \mathfrak{S}_d$, let $\widehat{R}_j(\sigma) = \bigl( \widehat{R}_j^{(1)}(\sigma), \dots, \widehat{R}_j^{(n)}(\sigma) \bigr)$ denote the ordinary least-squares residual vector obtained by regressing $X_j$ on its predecessors $A_j(\sigma)$ under $\sigma$.

The estimation scheme is then to maximize the empirical order objective
\[
\hat{\sigma}_n \in \operatorname*{\arg\max}_{\sigma \in \mathfrak{S}_d} \widehat{G}_n(\sigma),
\]
where
\begin{equation}
\widehat{G}_n(\sigma) \coloneqq \sum_{j = 1}^d \mathcal{W}_2\Bigl( \mathrm{std}\bigl( \frac{1}{n} \sum_{i = 1}^n \delta_{\widehat{R}_j^{(i)}(\sigma)} \bigr), \mathcal{N}(0, 1) \Bigr)^2. \label{eq:empirical-lingam-objective}
\end{equation}

Unlike the continuous orthogonal ICA problem, this finite, decomposable order objective can be tractably optimized exactly via dynamic programming.

\begin{remark}
Assumption~\ref{ass:causal-noise-normalization} implies that $\Sigma \coloneqq \mathbb{E}[X X^\top] \succ 0$. Therefore, for $n$ large enough, all sequential least-squares regressions are well-defined with high probability.
\end{remark}

\subsection{Convergence Guarantees}

\paragraph{Distribution-Free Case.}

Because the order space is finite for fixed $d$, uniform convergence follows by controlling finitely many empirical regressions and one-dimensional Wasserstein distances.

\begin{theorem}[Uniform Convergence of the Empirical Order Objective] \label{thm:dag-uniform-convergence}
Suppose Assumption~\ref{ass:causal-noise-normalization} holds. If all sequential least-squares regressions implicit in Equation~\eqref{eq:empirical-lingam-objective} are almost surely well-defined for $n$ large enough, then
\[
\max_{\sigma \in \mathfrak{S}_d} \bigl\vert \widehat{G}_n(\sigma) - G(\sigma) \bigr\vert \overset{\mathrm{a.s.}}{\to} 0.
\]
\end{theorem}

\bigbreak

The proof of Theorem~\ref{thm:dag-uniform-convergence} is provided in Appendix~\ref{pf:dag-uniform-convergence}.

\bigbreak

To convert uniform objective convergence into causal order recovery, we quantify the separation between causal and noncausal orders.

\begin{definition}[Oracle Order Separation Margin] \label{def:order-separation-margin}
When $\mathcal{I}^\star \neq \mathfrak{S}_d$, define the margin
\[
\Gamma \coloneqq \max_{\sigma \in \mathfrak{S}_d} G(\sigma) - \max_{\sigma \notin \mathcal{I}^\star} G(\sigma).
\]
\end{definition}

Theorem~\ref{thm:causal-order-identification} and finiteness of the order space imply that $\Gamma > 0$. This positive margin transfers uniform almost-sure objective convergence to exact recovery of a causal order.

\begin{corollary}[Consistency of Empirical Causal Orders] \label{cor:dag-order-consistency}
Suppose Assumptions~\ref{ass:causal-noise-normalization} and \ref{ass:causal-non-gaussianity} hold. Let $\hat{\sigma}_n$ be any maximizer of $\widehat{G}_n$ over $\mathfrak{S}_d$. If all sequential least-squares regressions implicit in Equation~\eqref{eq:empirical-lingam-objective} are almost surely well-defined for $n$ large enough, then
\[
\mathbb{P}\bigl( \hat{\sigma}_n \notin \mathcal{I}^\star \bigr) \leq \mathbb{P}\Bigl( \max_{\sigma \in \mathfrak{S}_d} \bigl\vert \widehat{G}_n(\sigma) - G(\sigma) \bigr\vert \geq \frac{\Gamma}{2} \Bigr) \to 0.
\]
\end{corollary}

\bigbreak

The proof of Corollary~\ref{cor:dag-order-consistency} is provided in Appendix~\ref{pf:dag-order-consistency}.

\bigbreak

Corollary~\ref{cor:dag-order-consistency} ensures that any empirical maximizer of the Wasserstein order objective belongs to the set of valid causal orders with probability tending to one. Thus, the procedure consistently recovers a valid causal ordering.

\paragraph{Sub-Gaussian Case.}

To obtain a quantitative rate, note that all $d 2^{d - 1}$ candidate regressions used in the oracle and empirical objectives (Definition~\ref{def:wasserstein-lingam-objective} and Equation~\eqref{eq:empirical-lingam-objective}) are determined by submatrices of $\Sigma$ and $\widehat{\Sigma} \coloneqq \frac{1}{n} \sum_{i = 1}^n X^{(i)}{X^{(i)}}^\top$, respectively. Hence, whenever $\widehat{\Sigma}$ is sufficiently close to $\Sigma$, all empirical residuals are uniformly close to their oracle counterparts. For instance, under a sub-Gaussian condition, the concentration of $\widehat{\Sigma}$ is such that the previous qualitative consistency statement admits the same quantitative rate as the empirical ICA objective. We now formalize the sub-Gaussian assumption and state the resulting consistency theorem.

\begin{definition}[Sub-Gaussian Random Variable]
A centered real-valued random variable $Z$ belongs to $\mathrm{SubG}(K)$ with $K > 0$ if
\[
\forall t \in \mathbb{R}, \, \mathbb{E}\bigl[ \exp(t Z) \bigr] \leq \exp\biggl( \frac{K^2 t^2}{2} \biggr).
\]
\end{definition}

\begin{assumption}[Sub-Gaussian Observations]
\label{ass:sub-gaussian-observations}
The random vector $X \in \mathbb{R}^d$ is centered with covariance matrix $\Sigma \succ 0$, and there exists $K \geq 1$ such that
\[
\forall \theta \in \mathbb{S}^{d - 1}, \, \theta^\top \Sigma^{-1/2} X \in \mathrm{SubG}(K).
\]
\end{assumption}

\begin{remark}
Assumption~\ref{ass:sub-gaussian-observations} implies, in particular, that Assumption~\ref{ass:moment} is met for all $p > 4$.
\end{remark}

\begin{theorem}[Sub-Gaussian Rate for Empirical Causal Orders] \label{thm:dag-order-sub-gaussian-rate}
Fix any $p > 4$. Suppose Assumptions~\ref{ass:causal-noise-normalization}, \ref{ass:causal-non-gaussianity} and~\ref{ass:sub-gaussian-observations} hold. Let $\hat{\sigma}_n$ be any maximizer of $\widehat{G}_n$ over $\mathfrak{S}_d$. Then there exists a universal constant $C_0 \geq 1$ and a constant $C_p(P) > 0$, depending only on $p$ and $\mathbb{E}[ \Vert \Sigma^{-1/2} X \Vert^p ]$ such that, for
\[
n \geq \max\bigl\{ 64 C_0^2 K^4 d, \, K^8 d, \, d \log(2n + 1)^{1 + 4/p} \bigr\},
\]
it holds
\[
\mathbb{E}\Bigl[ \max_{\sigma \in \mathfrak{S}_d} \bigl\vert \widehat{G}_n(\sigma) - G(\sigma) \bigr\vert \Bigr] \leq C_p(P) d^{5/4} n^{-1/4} \log(2n + 1)^{1/p + 1/4},
\]
and we have convergence in probability
\[
\mathbb{P}\bigl( \hat{\sigma}_n \notin \mathcal{I}^\star \bigr) \leq \frac{C_p(P) d^{5/4} n^{-1/4} \log(2n + 1)^{1/p + 1/4}}{\Gamma} \to 0.
\]
\end{theorem}

\bigbreak

The proof of Theorem~\ref{thm:dag-order-sub-gaussian-rate} is provided in Appendix~\ref{pf:dag-order-sub-gaussian-rate}.

\bigbreak

This theorem shows that the leading $d^{5/4} n^{-1/4}$ term is shared with the ICA analysis, since uniform estimation of the ordinary least-squares residuals contributes a $K^2 d^{3/2} n^{-1/2}$ term, which is absorbed by the leading term under $n \geq K^8 d$.

\section{Algorithms} \label{sec:algorithms}

This section describes how the optimization problems of Equations~\eqref{eq:empirical-ica-objective} and~\eqref{eq:empirical-lingam-objective} are solved in practice. While ICA is an optimization problem on the continuous group $\mathcal{O}_d(\mathbb{R})$, causal order search is essentially a finite combinatorial problem over $\mathfrak{S}_d$. In the following, we present four methods: one for ICA and three for causal discovery. Table~\ref{tab:complexity} summarizes our solvers and their costs.

\begin{table}[H]
\centering
\begin{tabular}{lllc}
\toprule
Task & Method & Cost & Global optimum \\
\midrule
ICA & Picard OT-ICA & $O(nd^2 + dn \log n)$ per iteration & --- \\
DAG & OT-ICA-LiNGAM & $O(nd^2 + dn \log n)$ per iteration & --- \\
DAG & Exhaustive OT-LiNGAM & $O\bigl( d 2^d (nd^2 + n \log n) \bigr)$ & yes (order) \\
DAG & Greedy OT-LiNGAM & $O(d^2 n \log n)$ & partial (order) \\
\bottomrule
\end{tabular}
\caption{Solvers proposed in the paper. Here, $d \geq 2$ is the number of variables, $n \geq d$ is the sample size, and an ``iteration'' is one pass of the corresponding method.}
\label{tab:complexity}
\end{table}

\subsection{Picard OT-ICA Solver}

Let $\mathbf{X} \in \mathbb{R}^{n \times d}$ contain the centered and whitened independent observations as rows. After whitening, ICA reduces to estimating the orthogonal rotation maximizing the empirical objective in Equation~\eqref{eq:empirical-ica-objective}. Following the Riemannian optimization framework of Picard ICA \citep{ablin2018faster}, we propose to optimize the criterion using limited-memory Broyden--Fletcher--Goldfarb--Shanno (L-BFGS, \cite{liu1989limited}) updates in the tangent spaces of the orthogonal group.

\paragraph{Initialization.}

The extensively studied and widely used FastICA algorithm \citep{hyvarinen1999fast} is used only as a warm start. Although it does not optimize the Wasserstein score, it typically provides a useful initial point on the orthogonal group.

\paragraph{Gradient Formulation.}

To formalize the optimization step, let $W \in \mathcal{O}_d(\mathbb{R})$ be the current unmixing matrix and set $\mathbf{Z} \coloneqq \mathbf{X} W^\top$. As we want to compute the Riemannian gradient of $\widehat{F}_n(W) = \sum_{j = 1}^d \mathcal{W}_2\bigl( (\pi_{W_j})_\sharp P_n, \mathcal{N}(0, 1) \bigr)^2$ on the orthogonal group $\mathcal{O}_d(\mathbb{R})$, we must parameterize the problem in such a way that any infinitesimal perturbation of $W$ remains an orthogonal matrix. One possibility is to use parameterizations along curves $\tau \mapsto \exp(\tau \mathcal{E}) W \in \mathcal{O}_d(\mathbb{R})$ (also called exponential retractions), where $\mathcal{E} \in \mathbb{R}^{d \times d}$ is an antisymmetric matrix. With the Frobenius inner product $\langle A, C \rangle_F \coloneqq \operatorname{tr}( A^\top C)$, differentiation along this curve gives
\[
\frac{d}{d\tau} \widehat{F}_n\bigl( \exp(\tau \mathcal{E}) W \bigr) \biggr\vert_{\tau = 0} = \bigl\langle \mathcal{E}, \mathcal{G}(W) \bigr\rangle_F,
\]
with $\mathcal{G}(W) \coloneqq \frac{1}{2} \bigl( \Psi(W)^\top \mathbf{Z} - \mathbf{Z}^\top \Psi(W) \bigr)$ and the score matrix $\Psi(W) \in \mathbb{R}^{n \times d}$ defined by
\[
\forall i \in \llbracket n \rrbracket, \, \forall j \in \llbracket d \rrbracket, \, \quad \Psi(W)_{ij} \coloneqq \frac{2}{n} (\mathbf{Z}_{ij} - q_{r_j(i)}),
\]
where for every $j \in \llbracket d \rrbracket$ and $i \in \llbracket n \rrbracket$, $r_j(i)$ denotes the rank of $\mathbf{Z}_{ij}$ among $\mathbf{Z}_{1j}, \dots, \mathbf{Z}_{nj}$ and $q_{r_j(i)}$ is defined in Lemma~\ref{lem:empirical-gaussian-wasserstein}.

This Riemannian gradient expression allows us to choose a local ascent direction for the empirical objective. Each update then follows the chosen parametric curve by applying the exponential \emph{retraction}, although this parameterization is not unique.

\paragraph{Solver.}

At iteration $t \geq 1$, L-BFGS combines the current gradient $\mathcal{G}_t \coloneqq \mathcal{G}(W_t)$ with at most $m \geq 1$ precomputed curvature pairs defined below. Given an ascent direction $\mathcal{D}_t$, an Armijo line search selects $\eta_t > 0$, and the iterate is updated by
\[
W_{t + 1} \coloneqq \exp\bigl( \eta_t\mathcal{D}_t \bigr)W_t.
\]

The associated curvature pair is $\mathcal{S}_t \coloneqq \eta_t \mathcal{D}_t$ and $\mathcal{V}_t \coloneqq \mathcal{G}_t - \mathcal{G}_{t + 1}$. This sign convention applies the standard L-BFGS recursion to the local minimization of $-\widehat{F}_n$. Pairs that do not satisfy $\langle \mathcal{S}_t, \mathcal{V}_t \rangle_F > 0$ are discarded. Algorithm~\ref{alg:picard-w2} summarizes the resulting procedure, which is implemented in the \href{https://pypi.org/project/otica/}{\texttt{otica}} package.

\bigbreak

\begin{algorithm}[H]
\DontPrintSemicolon
\caption{Picard OT-ICA: relative L-BFGS for Wasserstein ICA}
\label{alg:picard-w2}
\KwIn{whitened data $\mathbf{X} \in \mathbb{R}^{n \times d}$, memory size $m$, Armijo constant $c$, tolerance $\varepsilon$}
\KwOut{orthogonal unmixing $W \in \mathcal{O}_d(\mathbb{R})$}
\BlankLine
$W \gets \textsc{FastICA}(\mathbf{X})$ \quad $\mathbf{Z} \gets \mathbf{X} W^\top$\;
$\mathcal{G} \gets \frac{1}{2} \bigl( \Psi(W)^\top \mathbf{Z} - \mathbf{Z}^\top \Psi(W) \bigr)$\;
$\mathcal{M} \gets \emptyset$ \tcp*{L-BFGS memory of triples $(\mathcal{S}_\ell, \mathcal{V}_\ell, \rho_\ell)$}
\Repeat{$\Vert W - W_{\mathrm{old}} \Vert_F < \varepsilon$}{
  $\mathcal{D} \gets \textsc{TwoLoop}(\mathcal{G}, \mathcal{M})$ \tcp*{quasi-Newton ascent direction}
  \lIf{$\langle \mathcal{G}, \mathcal{D} \rangle_F \leq 0$}{$\mathcal{D} \gets \mathcal{G}$}
  $\eta \gets 1$\;
  \While{$\widehat{F}_n\bigl( \exp(\eta \mathcal{D}) W \bigr) < \widehat{F}_n(W) + c \, \eta \, \langle \mathcal{G}, \mathcal{D} \rangle_F$}{
    $\eta \gets \frac{\eta}{2}$\;
  }
  $W_{\mathrm{old}} \gets W$\;
  $W \gets \exp(\eta \mathcal{D}) W$ \quad $\mathbf{Z} \gets \mathbf{X} W^\top$\;
  $\mathcal{G}^{+} \gets \frac{1}{2} \bigl( \Psi(W)^\top \mathbf{Z} - \mathbf{Z}^\top \Psi(W) \bigr)$\;
  $\mathcal{S} \gets \eta \mathcal{D}$ \quad $\mathcal{V} \gets \mathcal{G} - \mathcal{G}^{+}$\;
  \lIf{$\langle \mathcal{S}, \mathcal{V} \rangle_F > 0$}{push $\bigl( \mathcal{S}, \mathcal{V}, \langle \mathcal{S}, \mathcal{V} \rangle_F^{-1} \bigr)$ onto $\mathcal{M}$, deleting the oldest if $\# \mathcal{M} > m$}
  $\mathcal{G} \gets \mathcal{G}^{+}$\;
}
\Return $W$\;
\end{algorithm}

\bigbreak

Computing $\Psi(W)$ requires sorting each of the $d$ recovered components, while forming $\Psi(W)^\top \mathbf{Z}$ requires a matrix product. Thus, excluding the line-search evaluations, one iteration costs $O(nd^2 + dn \log n)$.

\subsection{LiNGAM Solvers}

We again observe a centered data matrix $\mathbf{X} \in \mathbb{R}^{n \times d}$, whose rows are independent observations of the $d$ variables. We consider three strategies for estimating the causal order. The first combines the proposed OT-ICA estimator with the standard ICA-LiNGAM pipeline, while the other two optimize the proposed Wasserstein order objective directly. Once an order $\hat{\sigma}$ has been obtained, a second estimation step can be used to recover the weighted adjacency matrix $\widehat{B}(\hat{\sigma})$, for instance by fitting multiple linear regressions with adaptive Lasso, as in the \href{https://github.com/cdt15/lingam}{\texttt{lingam}} package.

\subsubsection{OT-ICA-LiNGAM} \label{subsec:ica-lingam}

A first practical strategy is to use an equivalent formulation of Model~\eqref{eq:lingam}, namely $\mathbf{X} = \boldsymbol{\varepsilon} (I_d - B^\star)^{-\top}$, which expresses the observations as linear mixtures of the independent structural-noise samples $\boldsymbol{\varepsilon} \in \mathbb{R}^{n \times d}$.

Leveraging Algorithm~\ref{alg:picard-w2}, the data can be unmixed using the OT-ICA solver after a whitening step. The resulting unmixing matrix then induces an estimate $\widehat{B}$, which is not guaranteed to be acyclic. As in standard ICA-LiNGAM, this estimate is post-processed by finding a suitable permutation through a linear-sum assignment problem to ``dagify'' $\widehat{B}$, which can be solved using the Hungarian algorithm \citep{shimizu2006linear}. This yields a direct analogue of the FastICA-based causal discovery implementation of the \href{https://github.com/cdt15/lingam}{\texttt{lingam}} package, but using our Wasserstein criterion. However, the main limitation, which also motivated the introduction of DirectLiNGAM by \cite{shimizu2011directlingam}, is that acyclicity is enforced only greedily and after estimation rather than being imposed as a constraint in the estimation problem itself.

\subsubsection{Exhaustive OT-LiNGAM}

To better reflect the optimization strategy described in Subsection~\ref{subsec:lingam-estim}, exact maximization of the empirical Wasserstein criterion $\widehat{G}_n$ from Equation~\eqref{eq:empirical-lingam-objective} can be carried out by dynamic programming. Indeed, the score is \emph{decomposable} as a sum over all nodes \citep{acid2003searching}, thus accommodating the methodology of \cite{silander2006simple}. However, the version used in this paper differs slightly from the full dynamic program because the objective is optimized over the set of all orders $\mathfrak{S}_d$ instead of the set of all DAGs.

\paragraph{Local Scores.}

Given a node $j \in \llbracket d \rrbracket$ and a candidate predecessor set $U \subseteq \llbracket d \rrbracket \setminus \bigl\{ j \bigr\}$, define the local empirical score
\[
\hat{g}_j(U) \coloneqq \mathcal{W}_2\Bigl( \mathrm{std}\bigl( \frac{1}{n} \sum_{i = 1}^n \delta_{\widehat{R}_j^{(i)}(U)} \bigr), \mathcal{N}(0, 1) \Bigr)^2,
\]
where $\widehat{R}_j(U)$ is the ordinary least-squares residual vector obtained by regressing $X_j$ on $\bigl\{ X_k : k \in U \bigr\}$. In practice, this score is evaluated exactly by solving the ordinary least-squares problem, standardizing the residuals, and computing the Wasserstein distance using Lemma~\ref{lem:empirical-gaussian-wasserstein}.

\paragraph{Dynamic Programming Recursion.}

For each nonempty subset $U \subseteq \llbracket d \rrbracket$, let $H(U)$ be the best empirical score over all orderings of the variables in $U$, and let $\tau^\star(U)$ be the corresponding last variable, or sink, in the restricted order. If $u$ is the sink of $U$, then its predecessors within the restricted order are exactly $U \setminus \{ u \}$, while the remaining variables must form an optimal order on $U \setminus \{ u \}$. This yields the recursion
\[
\begin{cases}
H(U) \coloneqq \max_{u \in U} \Bigl\{ H\bigl( U \setminus \{ u \} \bigr) + \hat{g}_u\bigl( U \setminus \{ u \} \bigr) \Bigr\}, \quad H(\emptyset) \coloneqq 0, \\
\tau^\star(U) \in \operatorname*{\arg\max}_{u \in U} \Bigl\{ H\bigl( U \setminus \{ u \} \bigr) + \hat{g}_u\bigl( U \setminus \{ u \} \bigr) \Bigr\}.
\end{cases}
\]

\paragraph{Causal Order Recovery.}

The causal order is recovered by backtracking through the stored sinks, using the recursion
\[
\begin{cases}
U_t \coloneqq U_{t + 1} \setminus \bigl\{ \hat{\sigma}(t + 1) \bigr\}, \quad U_d \coloneqq \llbracket d \rrbracket, \\
\hat{\sigma}(t) \coloneqq \tau^\star(U_t), \quad \hat{\sigma}(d) \coloneqq \tau^\star(U_d),
\end{cases} \quad \text{for } t = d - 1, \dots, 1.
\]

The recovered causal order is the permutation $\hat{\sigma} = \bigl( \hat{\sigma}(1), \dots, \hat{\sigma}(d) \bigr)$, obtained by successively selecting the stored sink of the current active set and placing it in the last remaining position.

\paragraph{Solver.}

Algorithm~\ref{alg:exhaustive-lingam} summarizes the resulting procedure, which is implemented in the \href{https://github.com/felixlaplante0/otlingam/}{\texttt{otlingam}} package. The implementation is highly optimized and parallelized and uses \href{https://github.com/numba/numba}{\texttt{numba}} to compile performance-critical code with LLVM.

\bigbreak

\begin{algorithm}[H]
\DontPrintSemicolon
\caption{Exhaustive OT-LiNGAM: subset dynamic programming for Wasserstein order search}
\label{alg:exhaustive-lingam}
\KwIn{centered data $\mathbf{X} \in \mathbb{R}^{n \times d}$}
\KwOut{causal order $\hat{\sigma}$ from source to sink}
\BlankLine
$H(\emptyset) \gets 0$\;
\For{$U \subseteq \llbracket d \rrbracket$ in increasing cardinality, $U \neq \emptyset$}{
  $h^\star \gets -\infty$ \quad $s^\star \gets -1$\;
  \For{$u \in U$ \tcp*{parallelizable}}{
    $h \gets H\bigl( U \setminus \{ u \} \bigr) + \hat{g}_u\bigl( U \setminus \{ u \} \bigr)$\;
    \lIf{$h > h^\star$}{$h^\star \gets h$ \quad $s^\star \gets u$}
  }
  $H(U) \gets h^\star$ \quad $\tau^\star(U) \gets s^\star$\;
}
$U \gets \llbracket d \rrbracket$\;
\For{$t = d, d - 1, \dots, 1$}{
  $\hat{\sigma}(t) \gets \tau^\star(U)$\;
  $U \gets U \setminus \bigl\{ \hat{\sigma}(t) \bigr\}$\;
}
\Return $\hat{\sigma}$\;
\end{algorithm}

\bigbreak

In the standard presentation of \cite{silander2006simple}, all $d 2^{d-1}$ local scores are precomputed and stored before running the subset dynamic program. Our implementation instead evaluates $\hat{g}_u\bigl(U \setminus \{ u \} \bigr)$ on demand within the sink recursion and stores only the subset scores $H(U)$ and selected sinks $\tau^\star(U)$. It therefore evaluates the same $d 2^{d-1}$ local residual scores but stores only $O(2^d)$ dynamic-programming states. Each local score costs $O(nd^2)$ operations for residual computation and $O(n\log n)$ for sorting. The resulting method is exponential in $d$ and is practically limited to about $d \approx 25$, but it returns the exact maximizer of Equation~\eqref{eq:empirical-lingam-objective} and remains far cheaper than enumerating all $d!$ causal orders, since $d! \gg d 2^{d-1}$ for large $d$.

\subsubsection{Greedy OT-LiNGAM}

A third greedy causal order estimation strategy consists of constructing the causal order from source to sink by repeatedly selecting the remaining variable whose standardized residual has the largest empirical $2$-Wasserstein distance to the Gaussian. In contrast to Exhaustive OT-LiNGAM, it does not compare complete orders. Instead, it maintains a residual matrix, initially equal to the centered data, and removes the linear effect of each selected source from every remaining variable. The main advantage of this strategy is its scalability, but since the procedure is irreversible, early errors cannot be corrected, which may make it more sensitive to noise and less stable than global maximization. Nonetheless, under the additional assumption that all structural noises have the same Wasserstein non-Gaussianity, Theorem~\ref{thm:greedy-equal-score}, provided alongside its proof in Appendix~\ref{pf:greedy-equal-score}, proves that the greedy algorithm returns a global maximizer in the oracle case. The method also remains applicable beyond this setting and can still give accurate causal orders when these non-Gaussianity levels only differ moderately.

\paragraph{Solver.}

At step $t \geq 1$, let $U_t$ denote the set of remaining variables and let $\widehat{R}_j(t) \in \mathbb{R}^n$ denote the current empirical residual of variable $j \in U_t$. After standardizing $\widehat{R}_j(t)$ by its empirical root mean square, the method computes the squared $2$-Wasserstein distance between its empirical distribution and the standard Gaussian, evaluated exactly using Lemma~\ref{lem:empirical-gaussian-wasserstein}. Once $j_t$ is selected, each remaining residual is updated by ordinary least-squares projection onto $\widehat{R}_{j_t}(t)$. Algorithm~\ref{alg:greedy-lingam} summarizes this procedure.

\bigbreak

\begin{algorithm}[H]
\DontPrintSemicolon
\caption{Greedy OT-LiNGAM: sequential Wasserstein source removal}
\label{alg:greedy-lingam}
\KwIn{centered data $\mathbf{X} \in \mathbb{R}^{n \times d}$}
\KwOut{causal order $\hat{\sigma}$ from source to sink}
\BlankLine
$\mathbf{R} \gets \mathbf{X}$ \quad $U \gets \llbracket d \rrbracket$\;
\For{$t \gets 1$ \KwTo $d$}{
  \For{$j \in U$}{
    $s_j \gets \sqrt{\frac{1}{n} \sum_{i = 1}^n \mathbf{R}_{ij}^2}$ \quad $\widetilde{\mathbf{R}}_{ij} \gets \frac{\mathbf{R}_{ij}}{s_j}$ for $i = 1, \dots, n$\;
    $h_j \gets \mathcal{W}_2\bigl( \frac{1}{n} \sum_{i = 1}^n \delta_{\widetilde{\mathbf{R}}_{ij}}, \mathcal{N}(0, 1) \bigr)^2$\;
  }
  $j_t \gets \operatorname*{\arg\max}_{j \in U} h_j$ \quad $\hat{\sigma}(t) \gets j_t$ \quad $U \gets U \setminus \bigl\{ j_t \bigr\}$\;
  \For{$k \in U$}{
    $\beta_k \gets \frac{\sum_{i = 1}^n \mathbf{R}_{i, j_t} \mathbf{R}_{ik}}{\sum_{i = 1}^n \mathbf{R}_{i, j_t}^2}$ \quad $\mathbf{R}_{ik} \gets \mathbf{R}_{ik} - \beta_k \mathbf{R}_{i, j_t}$ for $i = 1, \dots, n$\;
  }
}
\Return $\hat{\sigma}$\;
\end{algorithm}

\bigbreak

The greedy procedure performs $d$ sweeps rather than searching over all subsets, each costing $O(d n \log n)$, which reduces the order-recovery cost from exponential to polynomial.

\section{Experiments} \label{sec:experiments}

We evaluate Picard OT-ICA by its source-separation quality, unmixing-matrix recovery, and runtime. We evaluate Exhaustive OT-LiNGAM, Greedy OT-LiNGAM, and OT-ICA-LiNGAM by their causal order recovery and runtime. Each setting includes comparisons with multiple state-of-the-art baselines. The experiments are available in the accompanying repositories of the two companion packages, \href{https://github.com/felixlaplante0/otica/}{\texttt{otica}} and \href{https://github.com/felixlaplante0/otlingam/}{\texttt{otlingam}}. Additional computational environment and package-version details used to reproduce all experiments are provided in Appendix~\ref{sec:reproducibility}.

\subsection{ICA} \label{subsec:ica-experiments}

We first consider the general linear ICA model $X = AS$, where the coordinates of $S \in \mathbb{R}^d$ are independent, centered, and variance-normalized, and the entries of the mixing matrix $A$ are sampled independently from the standard Gaussian distribution, i.e.,
\[
\forall i, j \in \llbracket d \rrbracket, \, A_{ij} \overset{\mathrm{i.i.d.}}{\sim} \mathcal{N}(0, 1),
\]
which is almost surely invertible.

We compare Picard OT-ICA, implemented in the \href{https://github.com/felixlaplante0/otica/}{\texttt{otica}} package, with FastICA \citep{hyvarinen1999fast} from the \href{https://github.com/scikit-learn/scikit-learn}{scikit-learn implementation} \citep{pedregosa2011scikit}, using its default $\log$-cosh contrast. Both estimators center and whiten the same observations. Across the experiments, we evaluate the methods in terms of statistical performance as the sample size and dimension vary, robustness to nearly Gaussian latent components, and computational runtime.

For an estimated unmixing matrix $\widehat{A}^{-1}$, let $\widehat{G} \coloneqq \widehat{A}^{-1} A$ denote the gain matrix. We evaluate recovery using the Amari index \citep{poczos2005independent}
\[
\mathrm{Amari}(\widehat{G}) \coloneqq \frac{1}{2d(d - 1)} \Bigl[ \sum_{j = 1}^d \bigl( \frac{\sum_k \vert \widehat{G}_{jk} \vert}{\max_k \vert \widehat{G}_{jk} \vert} - 1 \bigr) + \sum_{k = 1}^d \bigl( \frac{\sum_j \vert \widehat{G}_{jk} \vert}{\max_j \vert \widehat{G}_{jk} \vert} - 1 \bigr) \Bigr].
\]

The Amari index is nonnegative and vanishes exactly when $\widehat{G}$ is a scaled permutation matrix, that is, when the sources are recovered up to scaling and permutation.

\paragraph{Statistical Performance.}

In this experiment, we assess how unmixing-matrix recovery varies with the sample size $n$ and the dimension $d$. We consider four variance-one source distributions: Laplace, uniform, centered exponential, and the calibrated uniform--exponential mixture defined below, all with zero mean and unit variance. For each source distribution and configuration, we independently generate $20$ source matrices and mixing matrices. We first vary $n \in \bigl\{ 100, 250, 500, 1{,}000, 1{,}500 \bigr\}$ while fixing $d = 8$, and then vary $d \in \bigl\{ 5, 10, 15, 20, 30 \bigr\}$ while fixing $n = 1{,}000$.

The fourth source distribution is constructed to expose a limitation of the default FastICA contrast without violating ICA identifiability. We define a uniform--exponential mixture
\[
\forall j \in \llbracket d \rrbracket, \, S_j(q) \overset{\mathrm{i.i.d.}}{\sim} q \mathrm{Unif}(-\sqrt{3}, \sqrt{3}) + (1 - q) \overline{\mathrm{Exp}}(1),
\]
where $\overline{\mathrm{Exp}}(1)$ denotes the centered exponential distribution.

Since both mixture components are centered and have unit variance, so does each source $S_j$ for any $j \in \llbracket d \rrbracket$. Given the $\log$-cosh contrast function $C(x) \coloneqq \log \cosh x$, the squared $\log$-cosh criterion used in the standard FastICA approximation to negentropy is given by
\[
J_C(Y) \coloneqq \Bigl( \mathbb{E}\bigl[ C(Y) \bigr] - \mathbb{E}\bigl[ C(Z) \bigr] \Bigr)^2, \quad Z \sim \mathcal{N}(0, 1).
\]

Writing $a_U \coloneqq \mathbb{E}\bigl[ C(U) \bigr]$, $a_E \coloneqq \mathbb{E}\bigl[ C(E) \bigr]$, and $a_Z \coloneqq \mathbb{E}\bigl[ C(Z) \bigr]$, with $U \sim \mathrm{Unif}(-\sqrt{3}, \sqrt{3})$, $E \sim \overline{\mathrm{Exp}}(1)$, and $Z \sim \mathcal{N}(0, 1)$, we notice
\[
a_E \approx 0.330 < a_Z \approx 0.375 < a_U \approx 0.401.
\]

For any $j \in \llbracket d \rrbracket$, since the density of the mixture $S_j(q)$ is the corresponding convex combination of the component densities, $\mathbb{E}\Bigl[ C\bigl( S_j(q) \bigr) \Bigr] = q \, a_U + (1 - q) \, a_E$ linearly interpolates between $a_U$ and $a_E$. Therefore, there exists a threshold $q \in (0, 1)$ satisfying
\[
q^\star \coloneqq \frac{a_Z - a_E}{a_U - a_E} \approx 0.624 \implies \forall j \in \llbracket d \rrbracket, \, \mathbb{E}\Bigl[ C\bigl( S_j(q^\star) \bigr) \Bigr] = \mathbb{E}\bigl[ C(Z) \bigr].
\]

Therefore, this distribution makes the $\log$-cosh negentropy criterion equal to zero at the true unmixing matrix, which becomes a global minimum rather than the maximum sought by FastICA. This construction is adversarial to the selected $\log$-cosh criterion and analogous examples can be obtained for other contrasts. We illustrate this phenomenon in dimension two in Appendix~\ref{sec:criterion-rotation}.

As shown in Figure~\ref{fig:varying-nd-amari-index}, for Laplace sources, the methods have comparable Amari indices in both settings. OT-ICA improves more clearly for all other source distributions. The calibrated mixture gives the sharpest contrast, since OT-ICA rapidly approaches a small Amari index as $n$ increases and remains accurate as $d$ grows, whereas FastICA remains far from the true unmixing. This is consistent with the distribution-free guarantees given by the Wasserstein objective and the deliberately vanishing $\log$-cosh contrast at the true unmixing directions.

\begin{figure}[t]
\centering
\includegraphics[width=\textwidth]{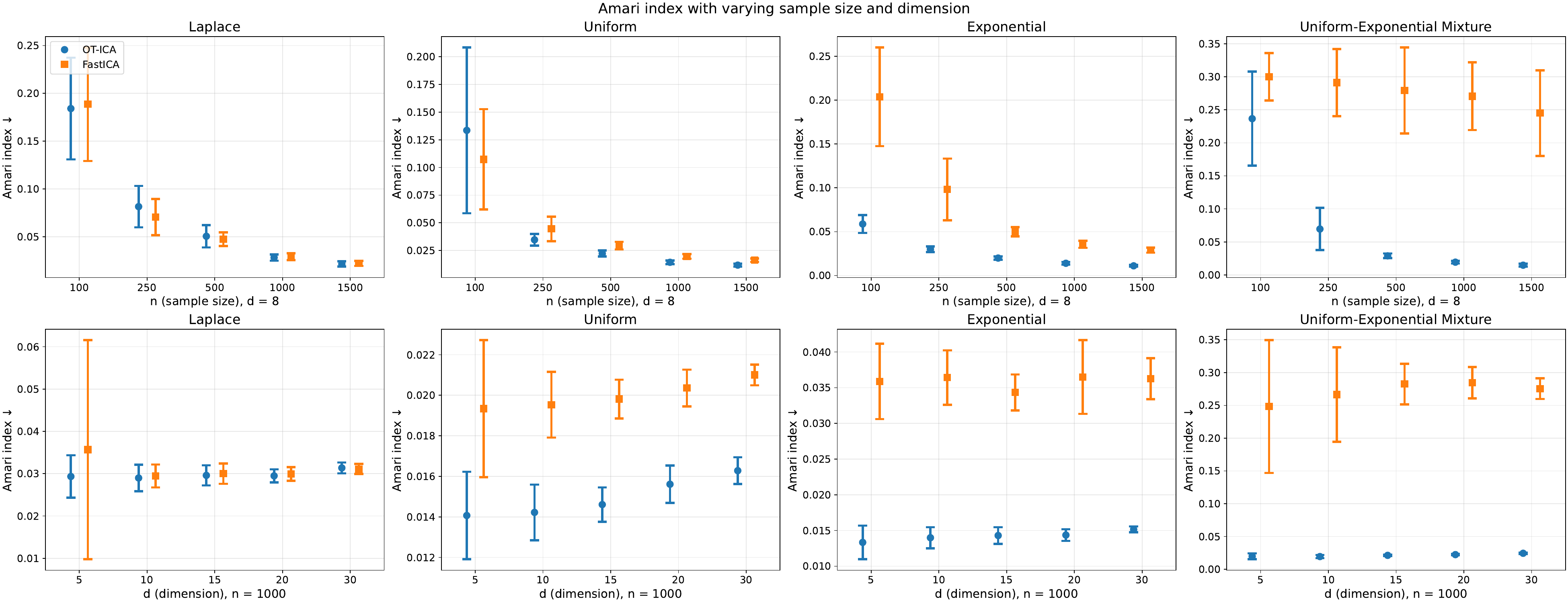}
\caption{ICA recovery for Laplace, uniform, centered-exponential, and calibrated uniform--exponential sources. Top: varying sample size with $d = 8$. Bottom: varying dimension with $n = 1{,}000$. Points show the mean Amari index over $20$ runs, and capped error bars show one standard deviation.}
\label{fig:varying-nd-amari-index}
\end{figure}

\paragraph{Varying Non-Gaussianity.}

We next interpolate each source distribution toward the Gaussian to evaluate the robustness to near-Gaussian data. Given an independent random vector $Z \sim \mathcal{N}(0, I_d)$, we replace $S$ by a convolution interpolating between the original distribution and the standard Gaussian, namely $S(\varepsilon) \coloneqq \sqrt{1 - \varepsilon} \, S + \sqrt{\varepsilon} \, Z$. We fix $n = 3{,}000$ and $d = 8$, vary $\varepsilon \in \bigl\{ 0, 0.05, 0.1, 0.2, 0.4, 0.7, 1 \bigr\}$, and average over $20$ independently generated datasets at each value. Figure~\ref{fig:gaussianity-amari-index} reports the resulting Amari indices. OT-ICA remains more accurate than FastICA away from the Gaussian endpoint, including for the calibrated mixture. Both methods deteriorate as $\varepsilon \to 1$. At $\varepsilon = 1$, all sources are Gaussian, so the unmixing matrix is not identifiable.

\begin{figure}[t]
\centering
\includegraphics[width=\textwidth]{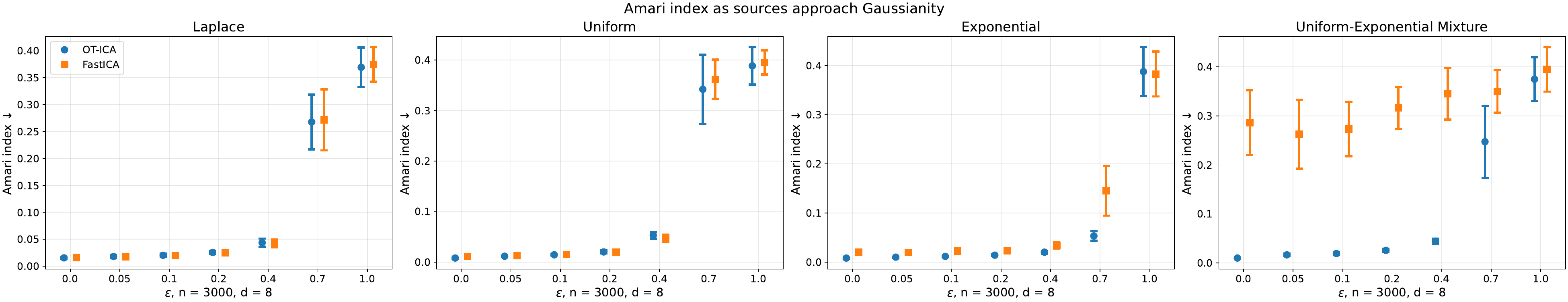}
\caption{ICA recovery as each source is replaced by $S(\varepsilon) = \sqrt{1 - \varepsilon} \, S + \sqrt{\varepsilon} \, Z$, with $n = 3{,}000$ and $d = 8$. Points show the mean Amari index over $20$ runs, and capped error bars show one standard deviation.}
\label{fig:gaussianity-amari-index}
\end{figure}

\paragraph{Runtime Performance.}

Finally, we measure fitting time using variance-one uniform sources. Each configuration is repeated $20$ times. The sample size experiment fixes $d = 8$ and varies $n$ from $250$ to $4{,}000$, while the dimension experiment fixes $n = 1{,}000$ and varies $d \in \bigl\{ 5, 10, 15, 20, 30 \bigr\}$. As detailed in Figure~\ref{fig:runtime-scaling-otica}, both methods exhibit similar scaling behavior, although OT-ICA is slower, typically by about one order of magnitude in the considered configurations. In absolute terms, both fitting times remain below one tenth of a second in most of the configurations considered here, making OT-ICA highly scalable.

\begin{figure}[t]
\centering
\includegraphics[width=0.7\textwidth]{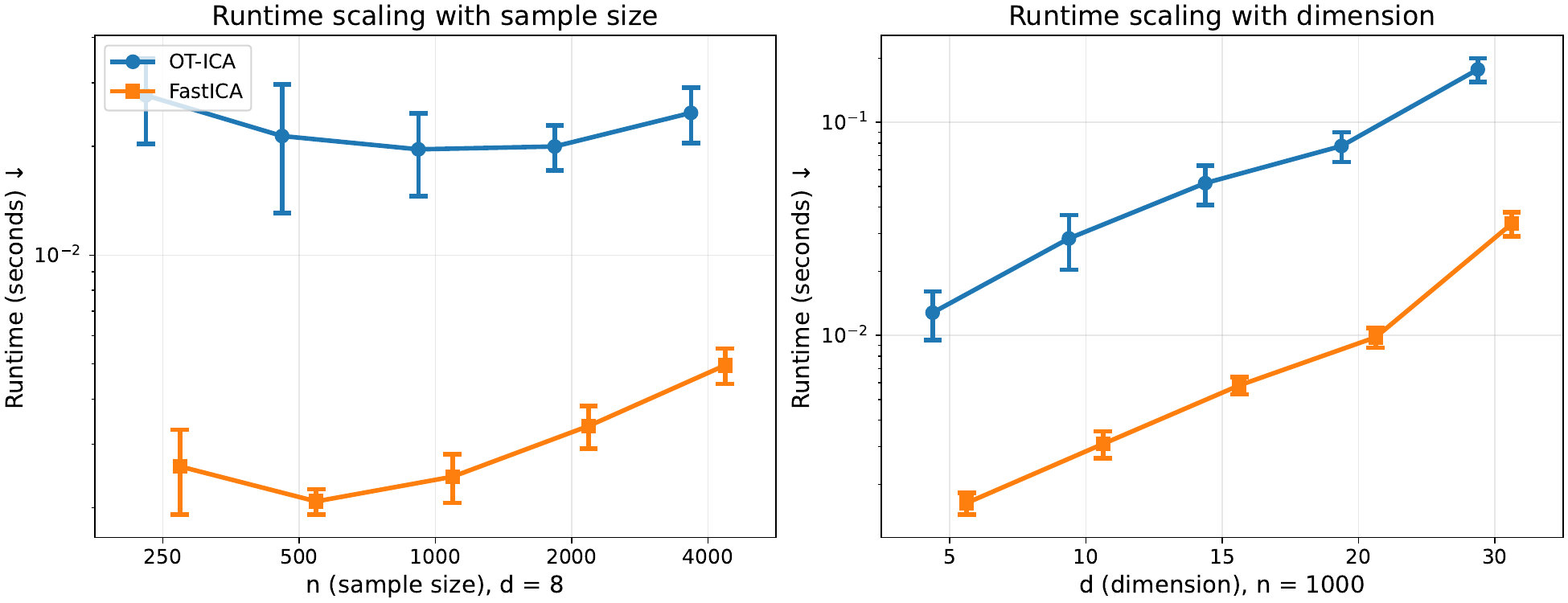}
\caption{ICA fitting time with variance-one uniform sources. Left: varying sample size with $d = 8$. Right: varying dimension with $n = 1{,}000$. Lines connect the mean runtime over $20$ runs, and capped error bars show one standard deviation.}
\label{fig:runtime-scaling-otica}
\end{figure}

\subsection{Causal Discovery}

For causal discovery, we consider linear structural equation models of the form $X = B^\star X + \varepsilon$, following Model~\eqref{eq:lingam}. We evaluate two standard benchmark graph families: Erd\H{o}s--R\'enyi (ER$k$, \cite{Erdos1984OnTE}) and Scale-Free (SF$k$). In the ER$k$ setting, every potential edge is included independently so that the resulting directed acyclic graph contains approximately $kd$ edges. In the SF$k$ setting, the graph is generated according to a preferential-attachment mechanism following the Barab\'asi--Albert model \citep{albert2002statistical}, while also containing approximately $kd$ edges. Thus, $k$ denotes the average in-degree of the graph in both families. After sampling the binary adjacency matrix $A_{ij}^\star \coloneqq \mathbb{I}\{B_{ij}^\star \neq 0\}$ from the chosen graph model, each nonzero coefficient is drawn independently with distribution $\mathrm{Unif}\bigl( [-2, -0.5] \cup [0.5, 2] \bigr)$.

All causal experiments compare the same six methods. Exhaustive OT-LiNGAM, Greedy OT-LiNGAM, and OT-ICA-LiNGAM are the proposed methods implemented in the \href{https://github.com/felixlaplante0/otlingam/}{\texttt{otlingam}} package. ICA-LiNGAM \citep{shimizu2006linear} and DirectLiNGAM \citep{shimizu2011directlingam} are taken from the \href{https://github.com/cdt15/lingam}{\texttt{lingam}} package. DAGMA \citep{bello2022dagma} is taken from its authors' \href{https://github.com/kevinsbello/dagma}{reference implementation}. This common comparison set is used for the sample size and dimension experiments, the noise-heterogeneity experiment, and the runtime experiment below. We evaluate the proposed methods and state-of-the-art competitors from four perspectives: statistical performance as the sample size and dimension vary, performance as graph density varies, robustness to nearly Gaussian structural noise, and runtime performance. Specific experimental settings use different distributions for the structural noise.

The main evaluation metric is the \emph{disorder}, which measures whether the estimated causal order reverses true causal edges. It is defined by
\[
\forall \sigma \in \mathfrak{S}_d, \, \mathrm{dis}(\sigma) = \#\bigl\{ (k, j) : B^\star_{jk} \neq 0 \text{ and } \sigma^{-1}(k) > \sigma^{-1}(j) \bigr\}.
\]

This metric ranges from $0$ to at most $d(d - 1) / 2$, and lower values indicate better causal order recovery. We also report two adjacency-matrix metrics. Denoting by $\widehat{A}_{jk} \coloneqq \mathbb{I}\{\widehat{B}_{jk} \neq 0\}$ the estimated binary adjacency matrix, the Structural Hamming Distance (SHD) and directed-edge F1 score are respectively defined by
\[
\mathrm{SHD}(\widehat{A}) \coloneqq \sum_{1 \leq j < k \leq d} \mathbb{I}\Bigl\{ \bigl( A^\star_{jk}, A^\star_{kj} \bigr) \neq \bigl( \widehat{A}_{jk}, \widehat{A}_{kj} \bigr) \Bigr\}, \quad \mathrm{F1}(\widehat{A}) \coloneqq \frac{2 \, \mathrm{TP}}{2 \, \mathrm{TP} + \mathrm{FP} + \mathrm{FN}},
\]
where $\mathrm{TP}$, $\mathrm{FP}$, and $\mathrm{FN}$ count true-positive, false-positive, and false-negative directed edges, respectively. Thus, a missing, extra, or reversed edge contributes one unit to SHD, which ranges from $0$ to $d(d - 1) / 2$, whereas F1 ranges from $0$ to $1$. Lower SHD and higher F1 indicate better recovery.

Although the SHD is often the main comparison metric in causal discovery \citep{acid2003searching, reisach2021beware}, we mainly focus on the disorder because it is fully agnostic to sparse parent selection and therefore isolates causal order recovery, while SHD and F1 also assess edge selection, which can be affected by other factors. Furthermore, the disorder of an estimated causal order $\hat{\sigma}$ is zero if and only if $\hat{\sigma} \in \mathcal{I}^\star$, i.e., if it is a true causal order for the given model, and thus directly reflects the convergence statement of Corollary~\ref{cor:dag-order-consistency}.

\paragraph{Statistical Performance.}

In this experiment, we assess how causal order recovery varies with the sample size $n$ and the dimension $d$. For each variable $j$, we first draw an independent scale $U_j \sim \mathrm{Unif}\bigl( [0.5, 2] \bigr)$, independently of all other sources of randomness. We then draw the structural noises independently according to
\[
\forall i \in \llbracket n \rrbracket, \, \forall j \in \llbracket d \rrbracket, \, \varepsilon_{ij} \coloneqq U_j Z_{ij}, \quad Z_{ij} \overset{\mathrm{i.i.d.}}{\sim} \mathrm{Lap}(0, 1).
\]

For each graph family and density, we average the disorder over $20$ independently generated datasets. We first vary $n \in \bigl\{ 100, 250, 500, 1{,}000, 1{,}500 \bigr\}$ while fixing $d = 8$, and then vary $d \in \bigl\{ 4, 6, 8, 10, 12 \bigr\}$ while fixing $n = 1{,}000$. Figure~\ref{fig:varying-nd-disorder} reports the results under the four graph configurations. Increasing the sample size reduces both the mean disorder and its variability. DAGMA fares worse than almost all other methods across the four graph families, whereas the methods that exploit non-Gaussianity approach perfect causal order recovery as $n$ increases. This difference may be understood from the fact that DAGMA is the only method in the comparison that does not exploit the non-Gaussianity of the structural noises. More precisely, the population least-squares problem underlying its score
\[
\min_{B \in \mathbb{R}^{d \times d} : B \text{ DAG}} \mathbb{E}\bigl[ \Vert X - BX \Vert^2 \bigr],
\]
does not generally admit $B^\star$ as a minimizer when the structural noises are heteroscedastic, even in the oracle setting \citep{loh2014high, chen2019causal, park2020identifiability, peters2014identifiability}. Among the remaining methods, Greedy OT-LiNGAM exhibits slightly higher finite-sample disorder and variability, while differences are otherwise small throughout the tested values of $(n, d)$. At $n = 1{,}000$, performance remains similarly stable as the dimension grows, with only a modest degradation for the greedy algorithm.

\begin{figure}[t]
\centering
\includegraphics[width=\textwidth]{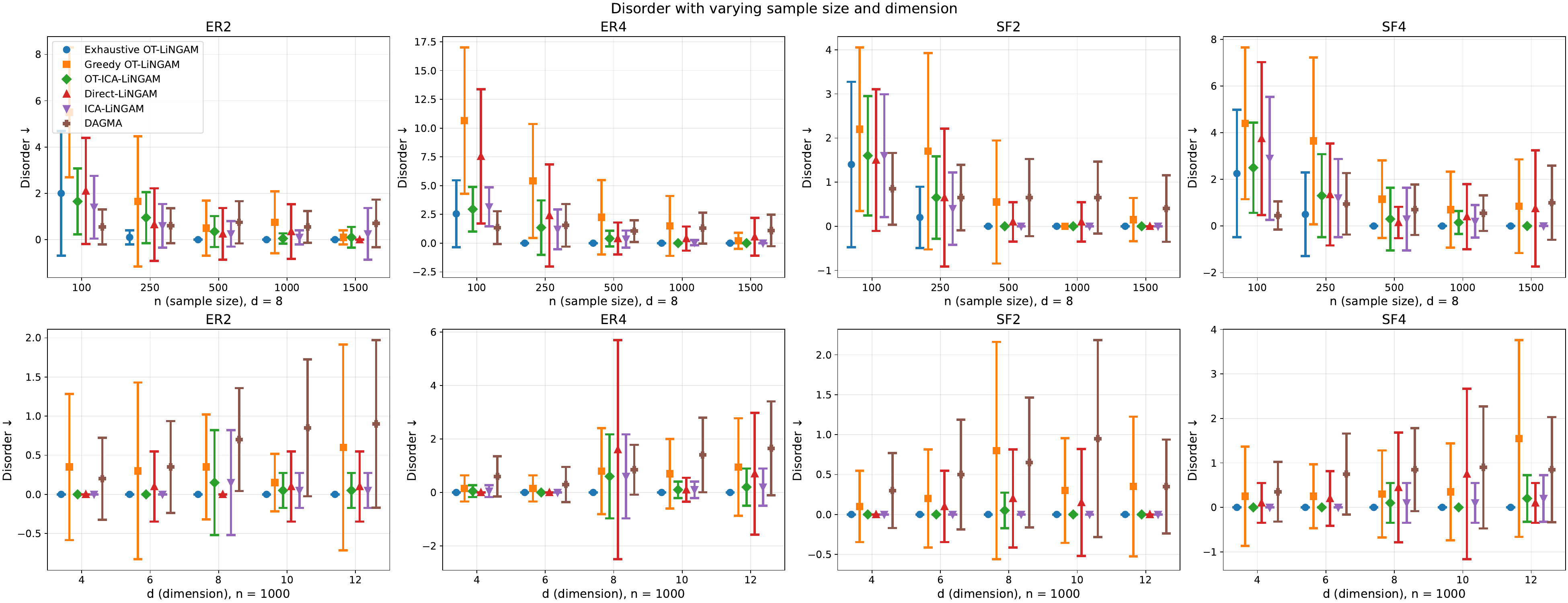}
\caption{Causal-order recovery on ER$2$, ER$4$, SF$2$, and SF$4$ graphs. Top: varying sample size with $d = 8$. Bottom: varying dimension with $n = 1{,}000$. Points show the mean disorder over $20$ runs, and capped error bars show one standard deviation.}
\label{fig:varying-nd-disorder}
\end{figure}

\paragraph{Varying non-Gaussianity.}

We next examine robustness as the structural noises become more heterogeneous and some component distributions become increasingly close to Gaussian. We fix $n = 3{,}000$ and $d = 8$, and set the $j$-th variable to be a standardized Student-$t$ noise with $\nu_j$ degrees of freedom
\[
\forall j \in \llbracket d \rrbracket, \, \nu_j \coloneqq \nu_\mathrm{min} + \frac{j - 1}{d - 1}\bigl( \nu_\mathrm{max} - \nu_\mathrm{min} \bigr).
\]

For each variable $j \in \llbracket d \rrbracket$, we independently draw $U_j \sim \mathrm{Unif}\bigl( [0.5, 2] \bigr)$, independently of the Student-$t$ variables, and sample the scaled, standardized noise as
\[
\forall i \in \llbracket n \rrbracket, \, \forall j \in \llbracket d \rrbracket, \, \varepsilon_{ij} \coloneqq U_j \sqrt{\frac{\nu_j - 2}{\nu_j}} Z_{ij}, \quad Z_{ij} \overset{\mathrm{ind.}}{\sim} \mathrm{Student}\text{-}t_{\nu_j}.
\]

In our setting, the minimum number of degrees of freedom is fixed at $\nu_\mathrm{min} = 2.5$, while we vary $\nu_\mathrm{max} \in \bigl\{ 2.5, 5, 10, 20, 40 \bigr\}$. Larger values of $\nu_\mathrm{max}$ make some noise components progressively closer to Gaussian while increasing the disparity in non-Gaussianity across variables. Figure~\ref{fig:noise-heterogeneity-disorder} reports the resulting disorder over $20$ independently generated datasets for each configuration. Exhaustive OT-LiNGAM maintains near-zero mean disorder across the full range of $\nu_\mathrm{max}$ and all but the SF$2$ graph configurations, surpassing all other models. The ICA-based methods and DirectLiNGAM degrade moderately as some noise components approach Gaussianity. One possible explanation is that these procedures do not optimize the order objective subject to acyclicity, although the experiment does not isolate this effect. This interpretation accords with the motivation for the direct procedure in \cite{shimizu2011directlingam}. OT-ICA-LiNGAM does, however, achieve noticeably better results than the standard FastICA-based ICA-LiNGAM, showing that the proposed ICA criterion improves causal order recovery. DAGMA fares well when $\nu_\mathrm{max}$ is large, but performs substantially worse at low values of $\nu_\mathrm{max}$. Greedy OT-LiNGAM is substantially more sensitive and deteriorates as noise heterogeneity increases. The deterioration occurs as the experiment increasingly violates the greedy method's additional equal-Wasserstein-non-Gaussianity assumption, whereas the exhaustive objective does not require this equality. This correspondence supports, but does not by itself establish, the assumption as the cause of the observed sensitivity.

\begin{figure}[t]
\centering
\includegraphics[width=\textwidth]{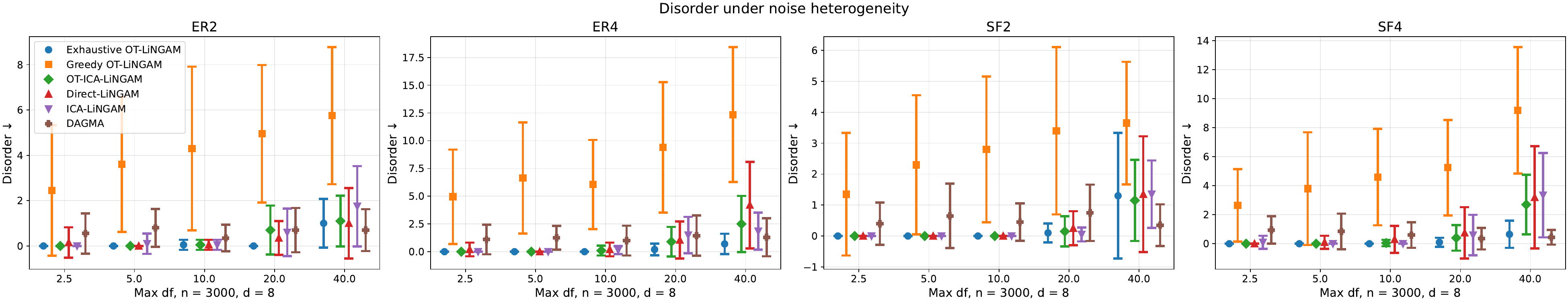}
\caption{Causal-order recovery under increasingly Gaussian structural noise on ER$2$, ER$4$, SF$2$, and SF$4$ graphs, with $n = 3{,}000$ and $d = 8$. Points show the mean disorder over $20$ runs, and capped error bars show one standard deviation. Larger $\nu_{\mathrm{max}}$ produces greater heterogeneity in non-Gaussianity across structural noises.}
\label{fig:noise-heterogeneity-disorder}
\end{figure}

\paragraph{Varying graph density.}

We next evaluate performance as the density of ER$k$ graphs increases. We fix $n = 1{,}000$ and $d = 8$, vary $k \in \bigl\{ 1, 2, 3, 4, 5, 6 \bigr\}$, and average each metric over $20$ independent runs. Figure~\ref{fig:varying-k-performance} reports disorder, SHD, and F1 score. Exhaustive OT-LiNGAM, OT-ICA-LiNGAM, and ICA-LiNGAM remain close to perfect recovery across the tested densities, with near-zero disorder and SHD and F1 scores close to one. Greedy OT-LiNGAM and DirectLiNGAM show greater variability as density increases, while DAGMA consistently gives larger SHD and lower F1 scores.

\begin{figure}[t]
\centering
\includegraphics[width=\textwidth]{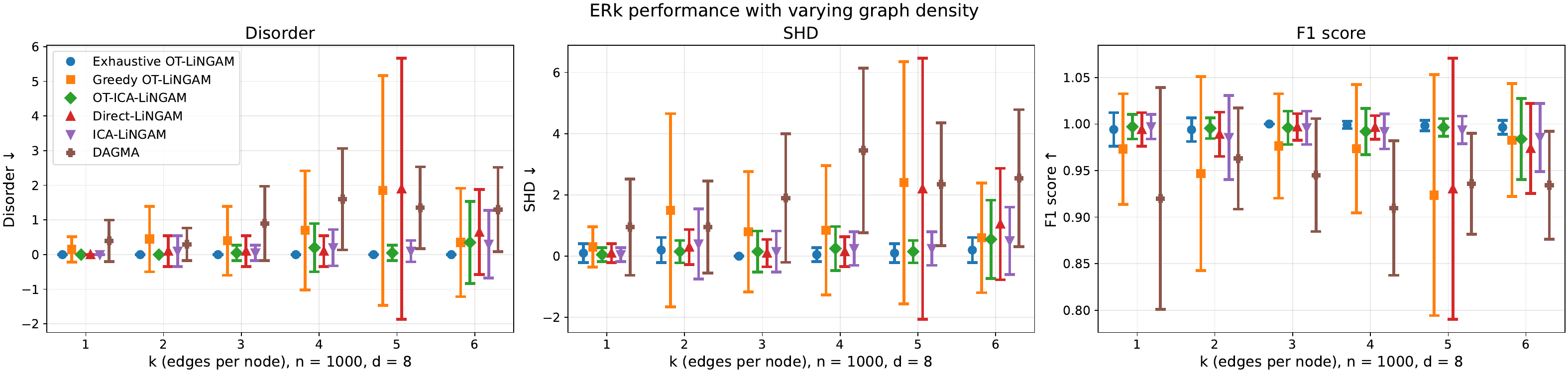}
\caption{Performance on ER$k$ graphs as density varies, with $n = 1{,}000$ and $d = 8$. Left: mean disorder. Center: mean SHD. Right: mean directed-edge F1 score. Points show means over $20$ runs, and capped error bars show one standard deviation.}
\label{fig:varying-k-performance}
\end{figure}

\paragraph{Runtime Performance.}

We finally evaluate runtime on ER$2$ graphs. We repeat each configuration independently for $10$ runs, varying $n$ from $250$ to $4{,}000$ at $d = 8$, and varying $d$ from $6$ to $20$ at $n = 1{,}000$. As shown in Figure~\ref{fig:runtime-scaling-lingam}, Greedy OT-LiNGAM and ICA-LiNGAM are the fastest methods over most configurations. Exhaustive OT-LiNGAM remains inexpensive at small $d$ but increases rapidly with dimension because its subset dynamic program has exponential complexity. OT-ICA-LiNGAM and DirectLiNGAM scale polynomially but are slower than the two fastest methods, while DAGMA has the largest runtime throughout this experiment.

\begin{figure}[t]
\centering
\includegraphics[width=0.7\textwidth]{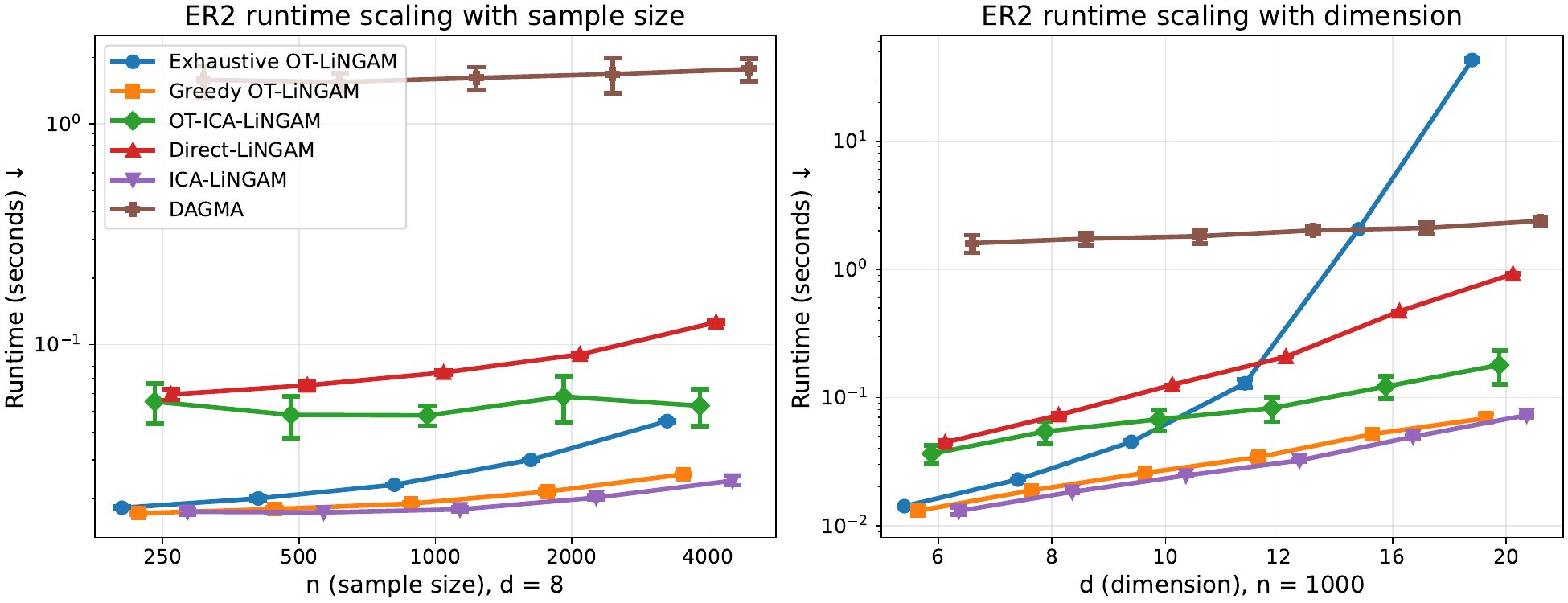}
\caption{Runtime on ER$2$ graphs. Left: varying sample size with $d = 8$. Right: varying dimension with $n = 1{,}000$. Lines connect the mean fitting time over $10$ runs, and capped error bars show one standard deviation.}
\label{fig:runtime-scaling-lingam}
\end{figure}

\section*{Conclusion and Perspectives}

We introduced a Wasserstein non-Gaussianity criterion for linear ICA and linear causal order recovery. The key mechanism is a strict comparison principle that yields identifiability of both the Wasserstein ICA objective and, through regression residuals, causal orders in linear non-Gaussian SEMs. The resulting framework uses a distributional score for source separation and causal discovery, avoiding hand-chosen ICA nonlinearities while retaining the standard non-Gaussian identifiability assumption. The same score also supports practical algorithms in both settings. For ICA, it can be optimized on the orthogonal group using a Picard-style solver, and causal orders can be estimated by exhaustive dynamic programming and greedy approximations.

Yet, the scope of these results is limited by the modeling assumptions. The analysis relies on linearity and on the independence of the latent sources or structural noises. These assumptions are standard in ICA and LiNGAM, but they exclude nonlinear causal mechanisms, feedback, and latent confounding. Furthermore, for ICA, the population objective is identifiable, but the empirical objective remains nonconvex on the orthogonal group, so practical Picard-style solvers are local optimizers. For causal discovery, Exhaustive OT-LiNGAM gives a true global optimum for the empirical order score, but its exponential dependence on the number of variables limits its use to moderate dimensions.

Future work may include nonlinear SEMs, latent-confounder models, and high-dimensional screening procedures that reduce the candidate pool before exact or greedy order search. Another direction is to improve Wasserstein ICA optimization using the piecewise-quadratic structure induced by projection ranks.

\section*{Funding}

The authors did not receive specific funding for this work.

\section*{Competing Interests}

The authors declare no competing interests relevant to the content of this article.

\bibliography{sources}

\newpage

\appendix

\section{Two-Dimensional Criterion Rotation} \label{sec:criterion-rotation}

In this section, we compare our Wasserstein criterion with the $\log$-cosh negentropy approximation used in FastICA. We use the same source distributions as in Subsection~\ref{subsec:ica-experiments}. For each distribution, we consider two independent, centered, variance-one sources $S_1$ and $S_2$ having that common distribution. Without loss of generality, take $w(0) = (1, 0)^\top$ as a true unmixing direction and rotate it through
\[
w(\theta) \coloneqq (\cos \theta, \sin \theta)^\top, \quad X(\theta) \coloneqq w(\theta)^\top S = \cos \theta \, S_1 + \sin \theta \, S_2, \quad \theta \in [-45^\circ, 45^\circ].
\]

For each angle, we evaluate our squared $2$-Wasserstein objective and the $\log$-cosh negentropy approximation. The empirical curves use $n = 3{,}000$ observations and show $20$ independently generated samples. Oracle curves illustrate the corresponding population objectives. As depicted in Figure~\ref{fig:criterion-rotation}, the empirical plug-in Wasserstein objective varies less across runs and follows its oracle curve more closely. In the final panel, the oracle $\log$-cosh objective reverses completely: the true unmixing direction at $\theta = 0$, marked by the vertical black dashed line, is a minimum of the negentropy approximation, whereas it remains the maximum of the Wasserstein objective.

\begin{figure}[H]
\centering
\includegraphics[width=\textwidth]{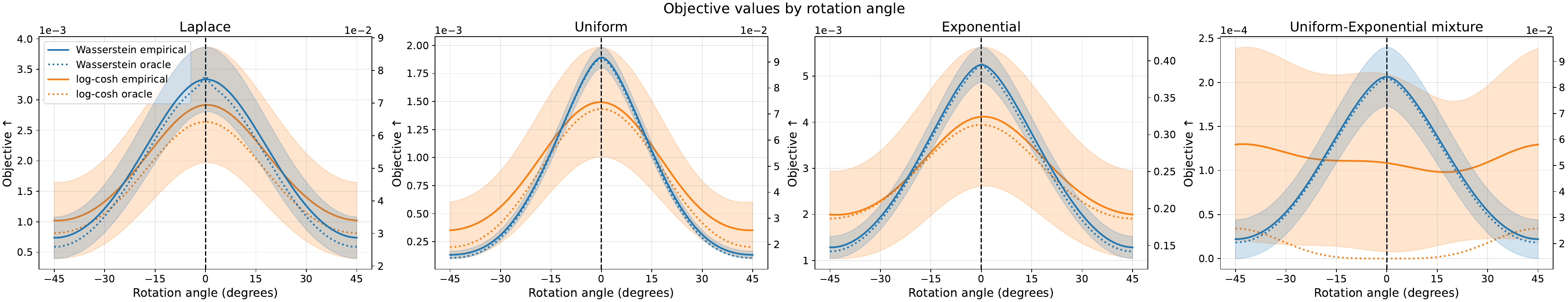}
\caption{Wasserstein and $\log$-cosh objectives along rotations of a true unmixing direction in dimension two for Laplace, uniform, centered-exponential, and calibrated uniform--exponential sources, with $n = 3{,}000$. The vertical black dashed line marks the true direction at $\theta = 0$. Solid lines show empirical means over $20$ runs, shaded regions show one standard deviation, and dotted lines show the oracle objectives.}
\label{fig:criterion-rotation}
\end{figure}

\section{Additional Results}

\begin{lemma}[Gaussian Distance Bound under Standardization] \label{lem:standardized-gaussian-bound}
Let $Z$ be a real-valued random variable such that $\mathbb{E}[Z^2] = 1$. Then
\[
\mathcal{W}_2\bigl( Z, \mathcal{N}(0, 1) \bigr)^2 \leq 2.
\]
\end{lemma}

\bigbreak

The proof of Lemma~\ref{lem:standardized-gaussian-bound} is provided in Appendix~\ref{pf:standardized-gaussian-bound}.

\bigbreak

\begin{lemma}[Exact Computation of the Empirical Gaussian Wasserstein Distance] \label{lem:empirical-gaussian-wasserstein}
Let $\Phi$ and $\phi$ denote the cumulative distribution and density functions of $\mathcal{N}(0, 1)$, respectively. For every $n \geq 1$, set $z_i \coloneqq \Phi^{-1}\bigl( \frac{i}{n} \bigr)$ for $i \in \llbracket n - 1 \rrbracket$, with $z_0 = -\infty$, $z_n = +\infty$, and $\phi(-\infty) = \phi(+\infty) = 0$, and define
\[
\forall i \in \llbracket n \rrbracket, \, q_i \coloneqq n \int_{\frac{i - 1}{n}}^{\frac{i}{n}} \Phi^{-1}(u) \, du = n \bigl( \phi(z_{i - 1}) - \phi(z_i) \bigr), \quad c_n \coloneqq 1 - \frac{1}{n} \sum_{i = 1}^n q_i^2.
\]

Then for every $x^{(1)}, \dots, x^{(n)} \in \mathbb{R}$, with order statistics $x^{(1:n)} \leq \cdots \leq x^{(n:n)}$ and empirical measure $\mu_n \coloneqq \frac{1}{n} \sum_{i = 1}^n \delta_{x^{(i)}}$,
\[
\mathcal{W}_2\bigl( \mu_n, \mathcal{N}(0, 1) \bigr)^2 = \frac{1}{n} \sum_{i = 1}^n \bigl( x^{(i:n)} - q_i \bigr)^2 + c_n.
\]

Moreover, the distance can be computed in $O(n \log n)$ time after precomputing $(q_i)_{i = 1}^n$ and $c_n$.
\end{lemma}

\bigbreak

The proof of Lemma~\ref{lem:empirical-gaussian-wasserstein} is provided in Appendix~\ref{pf:empirical-gaussian-wasserstein}.

\bigbreak

\begin{lemma}[Proportional Residual is Exogenous] \label{lem:prop-parents}
Suppose Assumption~\ref{ass:causal-noise-normalization} holds. Then, $X_j \propto \varepsilon_k$ for some $j, k \in \llbracket d \rrbracket$ in Model~\eqref{eq:lingam} if and only if $\mathrm{Pa}_{\mathcal{G}^\star}(j) = \emptyset$, i.e., $X_j$ is exogenous.
\end{lemma}

\bigbreak

The proof of Lemma~\ref{lem:prop-parents} is provided in Appendix~\ref{pf:prop-parents}.

\bigbreak

\begin{theorem}[Greedy Recovery under Equal Non-Gaussianity] \label{thm:greedy-equal-score}
Suppose Assumptions~\ref{ass:causal-noise-normalization} and~\ref{ass:causal-non-gaussianity} hold, and that all structural noises share the same Wasserstein non-Gaussianity score
\begin{equation}
\mathcal{W}_2\bigl( \mathrm{std}(\varepsilon_1), \mathcal{N}(0, 1) \bigr)^2 = \cdots = \mathcal{W}_2\bigl( \mathrm{std}(\varepsilon_d), \mathcal{N}(0, 1) \bigr)^2 = c > 0. \label{eq:greedy-equal-score}
\end{equation}

Then the oracle greedy Wasserstein rule returns a causal order of $\mathcal{G}^\star$.
\end{theorem}

\bigbreak

The proof of Theorem~\ref{thm:greedy-equal-score} is provided in Appendix~\ref{pf:greedy-equal-score}.

\begin{remark}[Greedy Failure Under Heterogeneous Noise]
Consider the two-variable model $X_1 = \varepsilon_1$, $X_2 = \beta X_1 + \varepsilon_2$, with $\beta \neq 0$, an instance of Model~\eqref{eq:lingam} with true order $\sigma^\star = (1, 2)$. Take the structural noises to be strongly unbalanced: $\varepsilon_1 \sim \mathcal{N}(0, 1)$ and $\varepsilon_2$ non-Gaussian, violating Equation~\eqref{eq:greedy-equal-score}.

At the first greedy step, the candidates are the raw standardized variables. The true source has score $\mathcal{W}_2\bigl( \mathrm{std}(X_1), \mathcal{N}(0, 1) \bigr)^2 = \mathcal{W}_2\bigl( \varepsilon_1, \mathcal{N}(0, 1) \bigr)^2 = 0$. The other candidate is a nontrivial mixture, hence
\[
\mathcal{W}_2\bigl( \mathrm{std}(X_2), \mathcal{N}(0, 1) \bigr)^2 > 0.
\]

Therefore, the greedy rule selects the wrong order $\hat{\sigma} = (2, 1)$.
\end{remark}

\section{Proofs}

\subsection{Proof of Lemma~\ref{lem:wasserstein-linear-combination}} \label{pf:wasserstein-linear-combination}

\begin{proof}
Since the variables have finite second moments, an optimal coupling $\pi_j$ between $S_j$ and $Z_j$ exists for each $j \in \llbracket d \rrbracket$. Then, consider the product coupling
\[
\pi \coloneqq \bigotimes_{j = 1}^d \pi_j.
\]

Let $(S_1', Z_1'), \dots, (S_d', Z_d')$ have joint distribution $\pi$. Since $\pi$ is a product coupling, the pairs $(S_j', Z_j')$ are independent and have the prescribed marginals. Since $S_j$ and $Z_j$ are centered, each difference $S_j' - Z_j'$ is centered. The product coupling therefore induces an admissible coupling between the two linear combinations. Expanding the square and using independence and centering to cancel the cross terms gives

\begin{align*}
\mathcal{W}_2\bigl( \sum_{j = 1}^d \alpha_j S_j, \sum_{j = 1}^d \alpha_j Z_j \bigr)^2 &\leq \mathbb{E}_\pi\Bigl[ \bigl( \sum_{j = 1}^d \alpha_j (S_j' - Z_j') \bigr)^2 \Bigr] \\
&= \sum_{j = 1}^d \alpha_j^2 \mathbb{E}_{\pi_j} \bigl[ (S_j' - Z_j')^2 \bigr] \\
&= \sum_{j = 1}^d \alpha_j^2 \mathcal{W}_2(S_j, Z_j)^2,
\end{align*}

where the final equality follows from the optimality of each $\pi_j$.
\end{proof}

\subsection{Proof of Lemma~\ref{lem:strict-wasserstein}} \label{pf:strict-wasserstein}

\begin{proof}
Let $Z_1, \dots, Z_d$ be independent standard Gaussian random variables. For each $j \in \llbracket d \rrbracket$, by Assumption~\ref{ass:ica-standardization}, $S_j \in \mathcal{P}_2(\mathbb{R})$, so Theorem~\ref{thm:brenier} ensures the existence of an almost everywhere unique monotone optimal transport map $T_j$ satisfying
\[
(T_j)_\sharp \mathcal{N}(0, 1) = S_j.
\]

Similarly, let $T$ be the monotone optimal transport map from $\mathcal{N}(0, 1)$ to the distribution of $\sum_{j = 1}^d \alpha_j S_j$. By Assumption~\ref{ass:ica-standardization}, the sources are independent, centered, and normalized, so Lemma~\ref{lem:wasserstein-linear-combination} gives the non-strict inequality; since $\Vert \alpha \Vert_2 = 1$, we have $\sum_{j = 1}^d \alpha_j Z_j \sim \mathcal{N}(0, 1)$.

Suppose for contradiction that equality holds, and let
\[
U \coloneqq \sum_{j = 1}^d \alpha_j Z_j, \quad V \coloneqq \sum_{j = 1}^d \alpha_j T_j(Z_j).
\]

The pair $(U, V)$ is the coupling used in the proof of Lemma~\ref{lem:wasserstein-linear-combination}. Equality implies that this coupling is optimal. Uniqueness of the optimal coupling with first marginal $\mathcal{N}(0, 1)$ gives $V = T(U)$ almost surely. Since $(Z_1, \dots, Z_d)$ has a strictly positive density on $\mathbb{R}^d$, we obtain
\begin{equation} \label{eq:pexider}
T\bigl(\sum_{j = 1}^d \alpha_j z_j\bigr) = \sum_{j = 1}^d \alpha_j T_j(z_j),
\end{equation}
for almost every $z \in \mathbb{R}^d$.

Since $\#\bigl\{ j \in \llbracket d \rrbracket : \alpha_j \neq 0 \bigr\} > 1$, choose distinct indices $j_1, j_2 \in \llbracket d \rrbracket$ such that $\alpha_{j_1}$ and $\alpha_{j_2}$ are nonzero. By Fubini's theorem, there exists $z_0 = (z_{0, j})_{j \notin \{ j_1, j_2 \}}$ such that Equation~\eqref{eq:pexider} holds for almost every $(z_{j_1}, z_{j_2}) \in \mathbb{R}^2$ after setting $z_j = z_{0, j}$ for all remaining indices. For this fixed $z_0$, Equation~\eqref{eq:pexider} becomes, almost everywhere
\[
T\bigl(\alpha_{j_1}z_{j_1} + \alpha_{j_2}z_{j_2} + c_0\bigr) = \alpha_{j_1}T_{j_1}(z_{j_1}) + \alpha_{j_2}T_{j_2}(z_{j_2}) + c_1
\]
where
\[
c_0 = \sum_{j \notin \{ j_1, j_2 \}} \alpha_j z_{0, j}, \quad c_1 = \sum_{j \notin \{ j_1, j_2 \}} \alpha_j T_j(z_{0, j}).
\]

After the changes of variables $x = \alpha_{j_1}z_{j_1}$ and $y = \alpha_{j_2}z_{j_2}$, define
\[
f(t) = T(t + c_0) - c_1, \quad g(x) = \alpha_{j_1} T_{j_1}\bigl( \frac{x}{\alpha_{j_1}} \bigr), \quad h(y) = \alpha_{j_2} T_{j_2}\bigl( \frac{y}{\alpha_{j_2}} \bigr).
\]

Then, we obtain the Pexider equation $f(x + y) = g(x) + h(y)$ for almost every $(x, y) \in \mathbb{R}^2$ with respect to Lebesgue measure. By standard results (see, e.g., \cite{de1966almost, aczel1989functional, reem2017remarks}), the measurable functions $f$, $g$, and $h$ must be affine almost everywhere. Therefore, $T_{j_1}(Z_{j_1})$ and $T_{j_2}(Z_{j_2})$ are affine images of standard Gaussian random variables. Their distributions, namely those of $S_{j_1}$ and $S_{j_2}$, are therefore Gaussian, contradicting Assumption~\ref{ass:ica-non-gaussianity}.
\end{proof}

\subsection{Proof of Theorem~\ref{thm:ica-identifiability}} \label{pf:ica-identifiability}

\begin{proof}
Fix $W \in \mathcal{O}_d(\mathbb{R})$. By Model~\eqref{eq:ica-model}, we have
\[
WX = W(W^\star)^{-1}S.
\]

Set $\widetilde{W} \coloneqq W(W^\star)^{-1}$. Since both $W$ and $W^\star$ are orthogonal, $\widetilde{W} \in \mathcal{O}_d(\mathbb{R})$. Thus, each recovered coordinate is a unit-norm linear combination of the true sources:
\[
\forall k \in \llbracket d \rrbracket, \, (WX)_k = (\widetilde{W}S)_k = \sum_{j = 1}^d \widetilde{W}_{kj}S_j.
\]

We now show that $F(W) \leq F(W^\star)$. Indeed, let $Z_1, \dots, Z_d$ be independent standard Gaussian variables. For every row $j \in \llbracket d \rrbracket$, we have
\[
\sum_{k = 1}^d \widetilde{W}_{jk} Z_k \sim \mathcal{N}\bigl( 0, \sum_{k = 1}^d (\widetilde{W}_{jk})^2 \bigr) = \mathcal{N}(0, 1).
\]

By Assumption~\ref{ass:ica-standardization}, applying Lemma~\ref{lem:wasserstein-linear-combination} to this row gives
\[
\mathcal{W}_2\bigl( \sum_{k = 1}^d \widetilde{W}_{jk} S_k, \mathcal{N}(0, 1) \bigr)^2 \leq \sum_{k = 1}^d (\widetilde{W}_{jk})^2 \mathcal{W}_2\bigl( S_k, \mathcal{N}(0, 1) \bigr)^2.
\]

Summing over all rows yields
\[
F(W) \leq \sum_{j = 1}^d \sum_{k = 1}^d (\widetilde{W}_{jk})^2 \mathcal{W}_2\bigl( S_k, \mathcal{N}(0, 1) \bigr)^2.
\]

Since $\widetilde{W}$ is orthogonal, its squared entries form a doubly stochastic matrix. Thus, we get
\[
F(W) \leq \sum_{k = 1}^d \Bigl\{ \mathcal{W}_2\bigl( S_k, \mathcal{N}(0, 1) \bigr)^2 \sum_{j = 1}^d (\widetilde{W}_{jk})^2 \Bigr\} = F(W^\star).
\]

Now, if $\widetilde{W}$ is a signed permutation matrix, every row selects exactly one source, and $WX$ is equivalent to $S$ up to the standard ICA indeterminacies. Hence $F(W) = F(W^\star)$.

If $\widetilde{W}$ is not a signed permutation matrix, at least one row contains two or more nonzero coefficients. By Assumptions~\ref{ass:ica-standardization} and~\ref{ass:ica-non-gaussianity}, Lemma~\ref{lem:strict-wasserstein} makes the inequality strict for that row, while the bounds for all other rows remain non-strict, hence $F(W) < F(W^\star)$.
\end{proof}

\subsection{Proof of Lemma~\ref{lem:empirical-gaussian-wasserstein}} \label{pf:empirical-gaussian-wasserstein}

\begin{proof}
Fix $n \geq 1$ and $x^{(1)}, \dots, x^{(n)} \in \mathbb{R}$. Since $n \geq 1$, the empirical measure and its order statistics are well-defined, and the one-dimensional quantile representation gives
\[
\mathcal{W}_2\bigl( \mu_n, \mathcal{N}(0, 1) \bigr)^2 = \sum_{i = 1}^n \int_{\frac{i - 1}{n}}^{\frac{i}{n}} \bigl( x^{(i:n)} - \Phi^{-1}(u) \bigr)^2 \, du.
\]

Expanding the squares and using $\int_0^1 \Phi^{-1}(u)^2 \, du = 1$ as well as the definition of $(q_i)_{i = 1}^n$ yields
\[
\mathcal{W}_2\bigl( \mu_n, \mathcal{N}(0, 1) \bigr)^2 = 1 + \frac{1}{n} \sum_{i = 1}^n \bigl( x^{(i:n)} \bigr)^2 - \frac{2}{n} \sum_{i = 1}^n x^{(i:n)}q_i.
\]

Completing the square gives
\[
\mathcal{W}_2\bigl( \mu_n, \mathcal{N}(0, 1) \bigr)^2 = \frac{1}{n} \sum_{i = 1}^n \bigl( x^{(i:n)} - q_i \bigr)^2 + \underbrace{1 - \frac{1}{n} \sum_{i = 1}^n q_i^2}_{c_n}.
\]

Furthermore, the values $q_i$ and $c_n$ depend only on $n$. Once they have been precomputed, sorting the observations costs $O(n \log n)$ and evaluating the sum costs $O(n)$, which proves the computational claim.
\end{proof}

\subsection{Proof of Lemma~\ref{lem:standardized-gaussian-bound}} \label{pf:standardized-gaussian-bound}

\begin{proof}
By definition, the Wasserstein distance is an infimum over all couplings. Therefore, taking $G \sim \mathcal{N}(0, 1)$ independent of $Z$, we have
\[
\mathcal{W}_2\bigl( Z, \mathcal{N}(0, 1) \bigr)^2 \leq \mathbb{E}\bigl[ (Z - G)^2 \bigr] = \underbrace{\mathbb{E}[Z^2] + \mathbb{E}[G^2]}_{2}  - 2\mathbb{E}[Z] \underbrace{\mathbb{E}[G]}_0 = 2. 
\]
\end{proof}

\subsection{Proof of Theorem~\ref{thm:uniform-convergence}} \label{pf:ica-uniform-convergence}

\begin{proof}
Fix $W \in \mathcal{O}_d(\mathbb{R})$. First, by the triangle inequality, we have
\begin{align*}
\bigl\vert \widehat{F}_n(W) - F(W) \bigr\vert &= \biggl\vert \sum_{j = 1}^d \Bigl\{ \mathcal{W}_2\bigl( (\pi_{W_j})_\sharp P_n, \mathcal{N}(0, 1) \bigr)^2 - \mathcal{W}_2\bigl( (\pi_{W_j})_\sharp P, \mathcal{N}(0, 1) \bigr)^2 \Bigr\} \biggr\vert \\
&\leq \sum_{j = 1}^d \Bigl\vert \mathcal{W}_2\bigl( (\pi_{W_j})_\sharp P_n, \mathcal{N}(0, 1) \bigr)^2 - \mathcal{W}_2\bigl( (\pi_{W_j})_\sharp P, \mathcal{N}(0, 1) \bigr)^2 \Bigr\vert.
\end{align*}

Further, by the reverse triangle inequality, for each $j \in \llbracket d \rrbracket$,
\[
\Bigl\vert \mathcal{W}_2\bigl( (\pi_{W_j})_\sharp P_n, \mathcal{N}(0, 1) \bigr) - \mathcal{W}_2\bigl( (\pi_{W_j})_\sharp P, \mathcal{N}(0, 1) \bigr) \Bigr\vert \leq \mathcal{W}_2\bigl( (\pi_{W_j})_\sharp P_n, (\pi_{W_j})_\sharp P \bigr).
\]

Since each row $W_j$ belongs to the unit sphere $\mathbb{S}^{d - 1}$, we can uniformly bound
\[
\mathcal{W}_2\bigl( (\pi_{W_j})_\sharp P_n, (\pi_{W_j})_\sharp P \bigr) \leq \sup_{\theta \in \mathbb{S}^{d - 1}} \mathcal{W}_2\bigl( (\pi_\theta)_\sharp P_n, (\pi_\theta)_\sharp P \bigr) \eqqcolon \widebar{\mathcal{W}}_2(P_n, P).
\]

By Theorem 3.5 of \cite{han2024maxsliced} and Assumption~\ref{ass:moment}, there exists a constant $C_p(P) > 0$, depending only on $p$ and $\mathbb{E}[ \Vert X \Vert^p ]$, such that
\[
\mathbb{E}\bigl[ \widebar{\mathcal{W}}_2(P_n, P)^2 \bigr] \leq C_p(P) \sqrt{\frac{d}{n}} \log(2n + 1)^{2/p + 1/2}.
\]
Jensen's inequality therefore gives
\[
\mathbb{E}\bigl[ \widebar{\mathcal{W}}_2(P_n, P) \bigr] \leq \sqrt{C_p(P)} d^{1/4} n^{-1/4} \log(2n + 1)^{1/p + 1/4}.
\]

\bigbreak

Now, set $\Delta_n \coloneqq \widebar{\mathcal{W}}_2(P_n, P)$. By Assumption~\ref{ass:ica-standardization} and Model~\eqref{eq:ica-model}, every projection of $P$ is centered with unit variance, so Lemma~\ref{lem:standardized-gaussian-bound} gives
\[
\begin{cases}
\mathcal{W}_2\bigl( (\pi_{W_j})_\sharp P, \mathcal{N}(0, 1) \bigr) \leq \sqrt{2}, \\
\Bigl\vert \mathcal{W}_2\bigl( (\pi_{W_j})_\sharp P_n, \mathcal{N}(0, 1) \bigr) - \mathcal{W}_2\bigl( (\pi_{W_j})_\sharp P, \mathcal{N}(0, 1) \bigr) \Bigr\vert \leq \Delta_n, \\
\mathcal{W}_2\bigl( (\pi_{W_j})_\sharp P_n, \mathcal{N}(0, 1) \bigr) + \mathcal{W}_2\bigl( (\pi_{W_j})_\sharp P, \mathcal{N}(0, 1) \bigr) \leq 2\sqrt{2} + \Delta_n.
\end{cases}
\]

Hence, using $\vert a^2 - b^2 \vert = \vert a + b \vert \vert a - b \vert$ for any $a, b \in \mathbb{R}$, we get
\[
\Bigl\vert \mathcal{W}_2\bigl( (\pi_{W_j})_\sharp P_n, \mathcal{N}(0, 1) \bigr)^2 - \mathcal{W}_2\bigl( (\pi_{W_j})_\sharp P, \mathcal{N}(0, 1) \bigr)^2 \Bigr\vert \leq 2\sqrt{2}\Delta_n + \Delta_n^2.
\]

This bound holds uniformly in $W$. Moreover, the sample size condition implies
\[
d^{1/4} n^{-1/4} \log(2n + 1)^{1/p + 1/4} \leq 1.
\]

Thus, the contribution of $d \mathbb{E}[\Delta_n^2]$ is bounded by the contribution of $d \mathbb{E}[\Delta_n]$. Summing the $d$ terms and enlarging $C_p(P)$ to absorb the numerical constants yields
\[
\mathbb{E}\Bigl[ \sup_{W \in \mathcal{O}_d(\mathbb{R})} \bigl\vert \widehat{F}_n(W) - F(W) \bigr\vert \Bigr] \leq C_p(P) d^{5/4} n^{-1/4} \log(2n + 1)^{1/p + 1/4}.
\]
\end{proof}

\subsection{Proof of Corollary~\ref{cor:ica-estimator-consistency}} \label{pf:ica-estimator-consistency}

\begin{proof}
Fix $\delta > 0$ such that $\mathcal{K}_\delta$ is nonempty. By Assumptions~\ref{ass:ica-standardization} and~\ref{ass:ica-non-gaussianity} and Theorem~\ref{thm:ica-identifiability}, $\Gamma(\delta) > 0$. Further, let
\[
Z_n \coloneqq \sup_{W \in \mathcal{O}_d(\mathbb{R})} \bigl\vert \widehat{F}_n(W) - F(W) \bigr\vert.
\]

Suppose $\rho(\widehat{W}_n) \geq \delta$. By definition, this implies $\widehat{W}_n \in \mathcal{K}_\delta$, and since $\widehat{W}_n$ maximizes $\widehat{F}_n$, we have
\begin{align*}
\Gamma(\delta) &\leq F(W^\star) - F(\widehat{W}_n) \\
&= F(W^\star) - \widehat{F}_n(W^\star) + \underbrace{\widehat{F}_n(W^\star) - \widehat{F}_n(\widehat{W}_n)}_{\leq 0} + \widehat{F}_n(\widehat{W}_n) - F(\widehat{W}_n) \\
&\leq 2Z_n.
\end{align*}

Therefore, by Assumption~\ref{ass:moment} and $n \geq d \log(2n + 1)^{1 + 4/p}$, Markov's inequality and Theorem~\ref{thm:uniform-convergence} give
\[
\mathbb{P}\bigl( \rho(\widehat{W}_n) \geq \delta \bigr) \leq \mathbb{P}\bigl( Z_n \geq \frac{\Gamma(\delta)}{2} \bigr) \leq \frac{2\mathbb{E}\bigl[ Z_n \bigr]}{\Gamma(\delta)} \leq \frac{C_p(P) d^{5/4} n^{-1/4} \log(2n + 1)^{1/p + 1/4}}{\Gamma(\delta)}.
\]

The right-hand side converges to zero, which proves convergence in probability. Finally, $\rho(\widehat{W}_n) \leq 2\sqrt{d}$ because both $\widehat{W}_n (W^\star)^{-1}$ and every element of $\mathcal{M}_d$ are orthogonal. For every $\delta > 0$ small enough
\[
\mathbb{E}\bigl[ \rho(\widehat{W}_n) \bigr] \leq \delta + 2\sqrt{d} \mathbb{P}\bigl( \rho(\widehat{W}_n) \geq \delta \bigr).
\]

Taking the limit superior as $n \to +\infty$ for fixed $\delta$ gives $\limsup_{n \to +\infty} \mathbb{E}\bigl[ \rho(\widehat{W}_n) \bigr] \leq \delta$. Letting $\delta \to 0$ proves convergence in expectation.
\end{proof}

\subsection{Proof of Lemma~\ref{lem:order-residuals-orthogonal}} \label{pf:order-residuals-orthogonal}

\begin{proof}
By Assumption~\ref{ass:causal-noise-normalization}, $\mathrm{std}(\varepsilon)$ is well-defined. By linearity and acyclicity of Model~\eqref{eq:lingam}, there exists an invertible matrix $\widetilde{B}^\star \in \mathrm{GL}_d(\mathbb{R})$ such that $X = \widetilde{B}^\star \mathrm{std}(\varepsilon)$. For a fixed order $\sigma \in \mathfrak{S}_d$, the ordinary least-squares residuals are linear in $X$, because there exists a matrix $H^\sigma$ such that
\[
R(\sigma) = H^\sigma X = H^\sigma \widetilde{B}^\star \mathrm{std}(\varepsilon).
\]

Up to a permutation, sequential least-squares residualization is the Gram--Schmidt orthogonalization applied to $X_{\sigma(1)}, \dots, X_{\sigma(d)}$. Therefore, the residuals are pairwise orthogonal, and by Assumption~\ref{ass:causal-noise-normalization}, after componentwise standardization
\[
\mathrm{Cov}\Bigl( \mathrm{std}\bigl( R(\sigma) \bigr) \Bigr) = I_d.
\]

Since $\mathrm{std}\bigl( R(\sigma) \bigr)$ is linear in $\mathrm{std}(\varepsilon)$, write $\mathrm{std}\bigl( R(\sigma) \bigr) = Q^\sigma \mathrm{std}(\varepsilon)$, which yields
\[
I_d = Q^\sigma (Q^\sigma)^\top \implies Q^\sigma \in \mathcal{O}_d(\mathbb{R}).
\]
\end{proof}

\subsection{Proof of Lemma~\ref{lem:prop-parents}} \label{pf:prop-parents}

\begin{proof}
First, suppose that $\mathrm{Pa}_{\mathcal{G}^\star}(j) = \emptyset$ for some $j \in \llbracket d \rrbracket$. Then, evidently, $X_j = \varepsilon_j$ therefore $X_j \propto \varepsilon_j$.

Conversely, suppose $X_j \propto \varepsilon_k$ for some $j, k \in \llbracket d \rrbracket$. Then, by Model~\eqref{eq:lingam}, letting $M^\star \coloneqq (I_d - B^\star)^{-1}$, $X$ can be written as
\[
X = M^\star \varepsilon.
\]

Furthermore, in any causal order $\sigma^\star \in \mathcal{I}^\star$, $\Pi^{\sigma^\star} B^\star (\Pi^{\sigma^\star})^\top$ is strictly lower triangular. As such, both $\Pi^{\sigma^\star} (I_d - B^\star) (\Pi^{\sigma^\star})^\top$ and $\Pi^{\sigma^\star} (I_d - B^\star)^{-1} (\Pi^{\sigma^\star})^\top$ are lower triangular with unit diagonal, so
\[
\forall h \in \llbracket d \rrbracket, \, M_{hh}^\star = 1.
\]

Thus, writing
\[
X_j = \sum_{h = 1}^d M_{jh}^\star \varepsilon_h = \varepsilon_j + \sum_{h \neq j} M_{jh}^\star \varepsilon_h,
\]
we have
\[
k = j \implies X_j \propto \varepsilon_j \implies X_j = \varepsilon_j \implies \sum_{h \neq j} B_{jh}^\star X_h \overset{\mathrm{a.s.}}{=} 0. 
\]

Hence, by Assumption~\ref{ass:causal-noise-normalization}, $X$ is nondegenerate and $B_{jh}^\star = 0$ for all $h \neq j$, which proves $\mathrm{Pa}_{\mathcal{G}^\star}(j) = \emptyset$.
\end{proof}

\subsection{Proof of Proposition~\ref{prop:permutation-causal}} \label{pf:permutation-causal}

\begin{proof}
If $Q^\sigma$ is a signed permutation matrix, then each standardized sequential residual is one standardized structural noise. For the first selected variable, there are no predecessors, hence the regression is on the empty span, and there exists $\ell_1 \in \llbracket d \rrbracket$ such that
\[
R_{\sigma(1)}(\sigma) = X_{\sigma(1)}, \quad \mathrm{std}\bigl( X_{\sigma(1)} \bigr) = \pm \mathrm{std}\bigl( \varepsilon_{\ell_1} \bigr).
\]

By Assumption~\ref{ass:causal-noise-normalization} and Lemma~\ref{lem:prop-parents}, a variable is proportional to one structural noise if and only if it is exogenous. Hence, $\mathrm{Pa}_{\mathcal{G}^\star}(\sigma(1)) = \emptyset$.

We now remove the effect of this exogenous variable from all remaining variables
\[
\forall j \in \llbracket d \rrbracket \setminus \{ \sigma(1) \}, \, X_j' \coloneqq X_j - \mathrm{proj}_{\mathrm{span}\bigl\{ X_{\sigma(1)} \bigr\}} X_j.
\]

By Assumption~\ref{ass:causal-noise-normalization}, the residual vector $X' = \bigl( X_j' : j \neq \sigma(1) \bigr)$ again satisfies a linear (possibly Gaussian) acyclic model on the remaining vertices. Projecting the remaining variables on their predecessors gives the same later residuals as in the original order, so the signed-permutation property is preserved. Applying the same argument inductively shows that $\sigma(2)$ is exogenous in the reduced model, then $\sigma(3)$ is exogenous in the next reduced model, until every variable has been selected. Therefore, every variable appears in $\sigma$ after its parents, which implies $\sigma \in \mathcal{I}^\star$.
\end{proof}

\subsection{Proof of Theorem~\ref{thm:causal-order-identification}} \label{pf:causal-order-identification}

\begin{proof}
Fix $\sigma \in \mathfrak{S}_d$. By Assumption~\ref{ass:causal-noise-normalization} and Lemma~\ref{lem:order-residuals-orthogonal}, we have
\[
\mathrm{std}\bigl( R_j(\sigma) \bigr) = \sum_{k = 1}^d (Q^\sigma)_{jk} \mathrm{std}(\varepsilon_k), \quad Q^\sigma \in \mathcal{O}_d(\mathbb{R}).
\]

Applying Lemma~\ref{lem:wasserstein-linear-combination} row-wise gives
\[
\mathcal{W}_2\Bigl( \mathrm{std}\bigl( R_j(\sigma) \bigr), \mathcal{N}(0, 1) \Bigr)^2 \leq \sum_{k = 1}^d (Q^\sigma)_{jk}^2 \mathcal{W}_2\bigl( \mathrm{std}(\varepsilon_k), \mathcal{N}(0, 1) \bigr)^2,
\]
and summing over all rows, using orthogonality, yields
\[
G(\sigma) \leq \sum_{j = 1}^d \sum_{k = 1}^d (Q^\sigma)_{jk}^2 \mathcal{W}_2\bigl( \mathrm{std}(\varepsilon_k), \mathcal{N}(0, 1) \bigr)^2 = \sum_{j = 1}^d \mathcal{W}_2\bigl( \mathrm{std}(\varepsilon_j), \mathcal{N}(0, 1) \bigr)^2.
\]

Now, if $\sigma \in \mathcal{I}^\star$, then every parent of $j \in \llbracket d \rrbracket$ belongs to the predecessor set of $j$ under $\sigma$. Since $\varepsilon_j$ is independent of all predecessors, the ordinary least-squares residual is
\[
R_j(\sigma) = \varepsilon_j \implies G(\sigma) = \sum_{j = 1}^d \mathcal{W}_2\bigl( \mathrm{std}(\varepsilon_j), \mathcal{N}(0, 1) \bigr)^2.
\]

Conversely, suppose that $\sigma$ attains this upper bound. By Assumptions~\ref{ass:causal-noise-normalization} and~\ref{ass:causal-non-gaussianity}, if some row of $Q^\sigma$ had at least two nonzero entries, then Lemma~\ref{lem:strict-wasserstein} would give
\[
G(\sigma) < \sum_{j = 1}^d \mathcal{W}_2\bigl( \mathrm{std}(\varepsilon_j), \mathcal{N}(0, 1) \bigr)^2,
\]
contradicting equality. Hence every row of $Q^\sigma$ has exactly one nonzero entry. Since $Q^\sigma$ is orthogonal, it is a signed permutation matrix, and Proposition~\ref{prop:permutation-causal} then implies $\sigma \in \mathcal{I}^\star$.
\end{proof}

\subsection{Proof of Theorem~\ref{thm:dag-uniform-convergence}} \label{pf:dag-uniform-convergence}

\begin{proof}
Fix $\sigma \in \mathfrak{S}_d$ and $j \in \llbracket d \rrbracket$. Since $d$ is fixed, the empirical ordinary least-squares projection of $X_j$ on its predecessors under $\sigma$ depends only on finitely many empirical second moments. By Assumption~\ref{ass:causal-noise-normalization} and the strong law of large numbers, these empirical moments converge almost surely to their true values. Hence, by assumption, the empirical regression coefficients are almost surely well-defined for $n$ large enough and converge almost surely to their oracle counterparts. Coupling observations with the same index and using the fact that the Wasserstein distance is the infimum over all couplings yields
\[
\mathcal{W}_2\Bigl( \frac{1}{n} \sum_{i = 1}^n \delta_{\widehat{R}_j^{(i)}(\sigma)}, \frac{1}{n} \sum_{i = 1}^n \delta_{R_j^{(i)}(\sigma)} \Bigr)^2 \leq \frac{1}{n} \sum_{i = 1}^n \bigl( \widehat{R}_j^{(i)}(\sigma) - R_j^{(i)}(\sigma) \bigr)^2 \overset{\mathrm{a.s.}}{\to} 0.
\]

By Assumption~\ref{ass:causal-noise-normalization}, the oracle residuals are i.i.d. with finite second moment, so their empirical distribution converges almost surely to the law of $R_j(\sigma)$ in $\mathcal{W}_2$ (see, e.g., Theorem 6.7 of \cite{villani2009optimal}). Their empirical second moment also converges to the positive population second moment, and the triangle inequality therefore gives
\[
\mathcal{W}_2\Bigl( \mathrm{std}\bigl( \frac{1}{n} \sum_{i = 1}^n \delta_{\widehat{R}_j^{(i)}(\sigma)} \bigr), \mathcal{N}(0, 1) \Bigr)^2 \overset{\mathrm{a.s.}}{\to} \mathcal{W}_2\Bigl( \mathrm{std}\bigl( R_j(\sigma) \bigr), \mathcal{N}(0, 1) \Bigr)^2.
\]

Summing over $j \in \llbracket d \rrbracket$ gives $\widehat{G}_n(\sigma) \overset{\mathrm{a.s.}}{\to} G(\sigma)$ for this fixed order. Because $\mathfrak{S}_d$ is finite, pointwise almost-sure convergence implies uniform almost-sure convergence over all orders
\[
\max_{\sigma \in \mathfrak{S}_d} \bigl\vert \widehat{G}_n(\sigma) - G(\sigma) \bigr\vert \overset{\mathrm{a.s.}}{\to} 0.
\]
\end{proof}

\subsection{Proof of Corollary~\ref{cor:dag-order-consistency}} \label{pf:dag-order-consistency}

\begin{proof}
First, assume that $\mathcal{I}^\star \neq \mathfrak{S}_d$ and set
\[
Z_n \coloneqq \max_{\sigma \in \mathfrak{S}_d} \bigl\vert \widehat{G}_n(\sigma) - G(\sigma) \bigr\vert.
\]

Suppose $\hat{\sigma}_n \notin \mathcal{I}^\star$. By the definition of $\Gamma$, for any $\sigma^\star \in \mathcal{I}^\star$ we have
\[
\Gamma \leq G(\sigma^\star) - G(\hat{\sigma}_n).
\]

Using the optimality of $\hat{\sigma}_n$ for $\widehat{G}_n$, we obtain
\begin{align*}
\Gamma &\leq G(\sigma^\star) - G(\hat{\sigma}_n) \\
&= G(\sigma^\star) - \widehat{G}_n(\sigma^\star) + \underbrace{\widehat{G}_n(\sigma^\star) - \widehat{G}_n(\hat{\sigma}_n)}_{\leq 0} + \widehat{G}_n(\hat{\sigma}_n) - G(\hat{\sigma}_n) \\
&\leq 2Z_n.
\end{align*}

Therefore, under the appropriate assumptions, Theorem~\ref{thm:dag-uniform-convergence} gives
\[
\mathbb{P}\bigl( \hat{\sigma}_n \notin \mathcal{I}^\star \bigr) \leq \mathbb{P}\bigl( Z_n \geq \frac{\Gamma}{2} \bigr) \to 0.
\]

If $\mathcal{I}^\star = \mathfrak{S}_d$, then $\hat{\sigma}_n \in \mathcal{I}^\star$ for every empirical maximizer by definition.
\end{proof}

\subsection{Proof of Theorem~\ref{thm:dag-order-sub-gaussian-rate}} \label{pf:dag-order-sub-gaussian-rate}

\begin{proof}
Let $\Sigma \coloneqq \mathbb{E}\bigl[ XX^\top \bigr]$, $\widehat{\Sigma} \coloneqq \frac{1}{n} \sum_{i = 1}^n X^{(i)}{X^{(i)}}^\top$. By Assumption~\ref{ass:causal-noise-normalization}, $\Sigma \succ 0$. For an order $\sigma \in \mathfrak{S}_d$ and a variable $j \in \llbracket d \rrbracket$, let $A_j(\sigma)$ denote the predecessor set of $j$, as in Definition~\ref{def:wasserstein-lingam-objective}. For a nonempty set $A = A_j(\sigma)$, define the oracle linear coefficient, the corresponding oracle residual and variance, as well as its standardized version, by
\[
\beta_{j, A} \coloneqq \Sigma_{A, A}^{-1} \Sigma_{A, j}, \quad R_{j, A} \coloneqq X_j - X_A^\top \beta_{j, A}, \quad s_{j, A}^2 \coloneqq \mathbb{E}[ R_{j, A}^2], \quad \widetilde{R}_{j, A} \coloneqq \frac{R_{j, A}}{s_{j, A}}.
\]

If $A = \emptyset$, we use the conventions
\[
\beta_{j, \emptyset} \coloneqq 0, \quad R_{j, \emptyset} \coloneqq X_j, \quad s_{j, \emptyset}^2 \coloneqq \mathbb{E}[ X_j^2], \quad \widetilde{R}_{j, \emptyset} \coloneqq \frac{X_j}{s_{j, \emptyset}}.
\]

Using the true residuals instead of the least-squares estimated ones, the corresponding objective is
\[
\widetilde{G}_n(\sigma) \coloneqq \sum_{j = 1}^d \mathcal{W}_2\Bigl( \frac{1}{n} \sum_{i = 1}^n \delta_{\widetilde{R}_{j, A_j(\sigma)}^{(i)}}, \mathcal{N}(0, 1) \Bigr)^2.
\]

By Lemma~\ref{lem:order-residuals-orthogonal}, the standardized oracle residuals satisfy $\mathrm{std}\bigl( R(\sigma) \bigr) = Q^\sigma \mathrm{std}(\varepsilon)$ for some $Q^\sigma \in \mathcal{O}_d(\mathbb{R})$. Moreover, $\mathrm{std}(\varepsilon)$ is an orthogonal transformation of $\Sigma^{-1/2} X$, so all its one-dimensional projections belong to $\mathrm{SubG}(K)$ by Assumption~\ref{ass:sub-gaussian-observations}, and $\mathbb{E}\bigl[ \Vert \mathrm{std}(\varepsilon) \Vert^p \bigr] = \mathbb{E}\bigl[ \Vert \Sigma^{-1/2} X \Vert^p \bigr] < +\infty$.

Denoting by $F^\varepsilon$ and $\widehat{F}_n^\varepsilon$ the population and empirical objectives of Definition~\ref{def:ica-objective-function} with $\mathrm{std}(\varepsilon)$ in place of $X$, and applying the proof of Theorem~\ref{thm:uniform-convergence}, we obtain
\begin{align*}
\mathbb{E}\Bigl[ \max_{\sigma \in \mathfrak{S}_d} \bigl\vert \widetilde{G}_n(\sigma) - G(\sigma) \bigr\vert \Bigr] &= \mathbb{E}\Bigl[ \max_{\sigma \in \mathfrak{S}_d} \bigl\vert \widehat{F}_n^\varepsilon(Q^\sigma) - F^\varepsilon(Q^\sigma) \bigr\vert \Bigr] \\
&\leq \mathbb{E}\Bigl[ \sup_{W \in \mathcal{O}_d(\mathbb{R})} \bigl\vert \widehat{F}_n^\varepsilon(W) - F^\varepsilon(W) \bigr\vert \Bigr] \\
&\leq C_p(P) d^{5/4} n^{-1/4} \log(2n + 1)^{1/p + 1/4}.
\end{align*}

It remains to control the replacement of the oracle residuals by their empirical least-squares counterparts. Given a target $j \in \llbracket d \rrbracket$ and candidate predecessor set $A \subseteq \llbracket d \rrbracket \setminus \{ j \}$, let $\beta_{j, A}$ and $\widehat{\beta}_{j, A}$ denote the oracle and empirical ordinary least-squares coefficients (which are well-defined assuming $\widehat{\Sigma}_{A, A}$ to be invertible), respectively. Since Equation~\eqref{eq:empirical-lingam-objective} is almost surely well-defined, the empirical coefficients and residuals exist. We also write $\mathbf{X}_A \in \mathbb{R}^{n \times \# A}$ for the sample matrix whose $i$-th row is ${X_A^{(i)}}^\top$ and $\mathbf{X}_j \coloneqq \bigl( X_j^{(1)}, \dots, X_j^{(n)} \bigr)^\top$. Define the oracle and empirical residual vectors by
\[
\boldsymbol{R}_{j, A} \coloneqq \mathbf{X}_j - \mathbf{X}_A \beta_{j, A}, \quad \widehat{\boldsymbol{R}}_{j, A} \coloneqq \mathbf{X}_j - \mathbf{X}_A \widehat{\beta}_{j, A}.
\]

For $A = \emptyset$, we use $\widehat{\beta}_{j, \emptyset} \coloneqq 0$ and $\boldsymbol{R}_{j, \emptyset} \coloneqq \widehat{\boldsymbol{R}}_{j, \emptyset} \coloneqq \mathbf{X}_j$.

First, to bound the difference between the \emph{unnormalized} residuals, notice that
\[
\widehat{\boldsymbol{R}}_{j, A} - \boldsymbol{R}_{j, A} = \bigl( \mathbf{X}_j - \mathbf{X}_A \widehat{\beta}_{j, A} \bigr) - \bigl( \mathbf{X}_j - \mathbf{X}_A \beta_{j, A} \bigr) = -\mathbf{X}_A (\widehat{\beta}_{j, A} - \beta_{j, A}).
\]

Therefore, taking its squared norm and using $\widehat{\Sigma}_{A, A} = \frac{1}{n} \mathbf{X}_A^\top \mathbf{X}_A$, we obtain
\begin{align}
\frac{1}{n} \Vert \widehat{\boldsymbol{R}}_{j, A} - \boldsymbol{R}_{j, A} \Vert^2 &= \frac{1}{n} \bigl( -\mathbf{X}_A (\widehat{\beta}_{j, A} - \beta_{j, A}) \bigr)^\top \bigl( -\mathbf{X}_A (\widehat{\beta}_{j, A} - \beta_{j, A}) \bigr) \notag \\
&= (\widehat{\beta}_{j, A} - \beta_{j, A})^\top \bigl( \frac{1}{n} \mathbf{X}_A^\top \mathbf{X}_A \bigr) (\widehat{\beta}_{j, A} - \beta_{j, A}) \notag \\
&= (\widehat{\beta}_{j, A} - \beta_{j, A})^\top \widehat{\Sigma}_{A, A} (\widehat{\beta}_{j, A} - \beta_{j, A}) \label{eq:residuals}.
\end{align}

Let $E \coloneqq \widehat{\Sigma} - \Sigma$, and let $v_{j, A} \in \mathbb{R}^d$ equal $1$ at coordinate $j$, $-\beta_{j, A}$ on $A$, and $0$ elsewhere. Then, directly from the definitions of $\beta_{j, A}$ and $\widehat{\beta}_{j, A}$,
\[
s_{j, A}^2 \coloneqq \mathbb{E}[ R_{j, A}^2 ] = v_{j, A}^\top \Sigma v_{j, A}, \quad \widehat{\Sigma}_{A, A} (\widehat{\beta}_{j, A} - \beta_{j, A}) = E_{A, :} v_{j, A}.
\]

Therefore, Equation~\eqref{eq:residuals} gives
\[
\frac{1}{n} \Vert \widehat{\boldsymbol{R}}_{j, A} - \boldsymbol{R}_{j, A} \Vert^2 = (E_{A, :} v_{j, A})^\top \widehat{\Sigma}_{A, A}^{-1} (E_{A, :} v_{j, A}) = v_{j, A}^\top (E_{:, A} \widehat{\Sigma}_{A, A}^{-1} E_{A, :}) v_{j, A}.
\]
Define $\delta_n \coloneqq \bigl\Vert \Sigma^{-1/2} \bigl( \widehat{\Sigma} - \Sigma \bigr) \Sigma^{-1/2} \bigr\Vert_\mathrm{op}$. If $\delta_n < 1$, then
\[
(1 - \delta_n) \Sigma \preceq \widehat{\Sigma}, \quad E_{:, A} \Sigma_{A, A}^{-1} E_{A, :} \preceq E \Sigma^{-1} E \preceq \Sigma^{1/2} (\Sigma^{-1/2} E \Sigma^{-1/2})^2 \Sigma^{1/2} \preceq \delta_n^2 \Sigma,  
\]
hence
\[
\widehat{\Sigma}_{A, A}^{-1} \preceq \frac{1}{1 - \delta_n} \Sigma_{A, A}^{-1}, \quad E_{:, A} \widehat{\Sigma}_{A, A}^{-1} E_{A, :} \preceq \frac{\delta_n^2}{1 - \delta_n} \Sigma.
\]

It now follows that
\[
\frac{1}{n} \Vert \widehat{\boldsymbol{R}}_{j, A} - \boldsymbol{R}_{j, A} \Vert^2 \leq \frac{\delta_n^2}{1 - \delta_n} s_{j, A}^2.
\]

To handle the normalized versions, which generally do not coincide, we must uniformly control the estimation error of $s_{j, A}$. Define
\[
\hat{s}_{j, A} \coloneqq \frac{1}{\sqrt{n}} \Vert \widehat{\boldsymbol{R}}_{j, A} \Vert, \quad \mathrm{std}(\widehat{\boldsymbol{R}}_{j, A}) \coloneqq \frac{\widehat{\boldsymbol{R}}_{j, A}}{\hat{s}_{j, A}}, \quad \widetilde{\boldsymbol{R}}_{j, A} \coloneqq \frac{\boldsymbol{R}_{j, A}}{s_{j, A}}.
\]

Since $\frac{1}{n} \Vert \boldsymbol{R}_{j, A} \Vert^2 = v_{j, A}^\top \widehat{\Sigma} v_{j, A}$ and $s_{j, A}^2 = v_{j, A}^\top \Sigma v_{j, A}$, we have
\begin{equation}
\biggl\vert \frac{\Vert \boldsymbol{R}_{j, A} \Vert}{\sqrt{n} s_{j, A}} - 1 \biggr\vert \leq \biggl\vert \frac{\Vert \boldsymbol{R}_{j, A} \Vert^2}{n s_{j, A}^2} - 1 \biggr\vert = \frac{\bigl\vert v_{j, A}^\top (\widehat{\Sigma} - \Sigma) v_{j, A} \bigr\vert}{v_{j, A}^\top \Sigma v_{j, A}} \leq \delta_n. \label{eq:s-ratio}
\end{equation}

Moreover, the reverse triangle inequality gives
\[
\biggl\vert \frac{\hat{s}_{j, A}}{s_{j, A}} - \frac{\Vert \boldsymbol{R}_{j, A} \Vert}{\sqrt{n} s_{j, A}} \biggr\vert = \biggl\vert \frac{\Vert \widehat{\boldsymbol{R}}_{j, A} \Vert}{\sqrt{n} s_{j, A}} - \frac{\Vert \boldsymbol{R}_{j, A} \Vert}{\sqrt{n} s_{j, A}} \biggr\vert
\leq \frac{1}{\sqrt{n} s_{j, A}} \Vert \widehat{\boldsymbol{R}}_{j, A} - \boldsymbol{R}_{j, A} \Vert \leq \frac{\delta_n}{\sqrt{1 - \delta_n}}.
\]

Combining the preceding inequality with Equation~\eqref{eq:s-ratio}, and using the triangle inequality, we get
\[
\biggl\vert \frac{\hat{s}_{j, A}}{s_{j, A}} - 1 \biggr\vert \leq \biggl\vert \frac{\Vert \boldsymbol{R}_{j, A} \Vert}{\sqrt{n} s_{j, A}} - 1 \biggr\vert + \biggl\vert \frac{\hat{s}_{j, A}}{s_{j, A}} - \frac{\Vert \boldsymbol{R}_{j, A} \Vert}{\sqrt{n} s_{j, A}} \biggr\vert \leq \delta_n + \frac{\delta_n}{\sqrt{1 - \delta_n}} \leq \frac{2\delta_n}{\sqrt{1 - \delta_n}}.
\]

On the event $\delta_n \leq 1/4$, the preceding bound and Equation~\eqref{eq:s-ratio} give
\[
\frac{\hat{s}_{j, A}}{s_{j, A}} \geq 1 - \frac{1}{\sqrt{3}}, \quad \frac{\Vert \boldsymbol{R}_{j, A} \Vert}{\sqrt{n} s_{j, A}} \leq 5/4.
\]

Coupling observations with the same index and using the fact that the Wasserstein distance is the infimum over all couplings therefore yields
\begin{align*}
\mathcal{W}_2\Bigl( \frac{1}{n} \sum_{i = 1}^n \delta_{\mathrm{std}(\widehat{\boldsymbol{R}}_{j, A})_i}, \frac{1}{n} \sum_{i = 1}^n \delta_{\widetilde{R}_{j, A}^{(i)}} \Bigr) &\leq \frac{1}{\sqrt{n}} \bigl\Vert \mathrm{std}(\widehat{\boldsymbol{R}}_{j, A}) - \widetilde{\boldsymbol{R}}_{j, A} \bigr\Vert \\
&\leq \frac{s_{j, A}}{\hat{s}_{j, A}} \frac{\Vert \widehat{\boldsymbol{R}}_{j, A} - \boldsymbol{R}_{j, A} \Vert}{\sqrt{n} s_{j, A}} + \frac{\Vert \boldsymbol{R}_{j, A} \Vert}{\sqrt{n} s_{j, A}} \biggl\vert \frac{s_{j, A}}{\hat{s}_{j, A}} - 1 \biggr\vert \\
&\leq \frac{1}{1 - \frac{1}{\sqrt{3}}} \Biggl( \frac{\delta_n}{\sqrt{1 - \delta_n}} + \frac{5}{4}\frac{2\delta_n}{\sqrt{1 - \delta_n}} \Biggr) \leq 10\delta_n.
\end{align*}

On this event, the second moments of both empirical residual distributions are bounded by a universal constant, and hence so are their Wasserstein distances to $\mathcal{N}(0, 1)$. Fix $\sigma \in \mathfrak{S}_d$. Applying, as in the proof of Theorem~\ref{thm:uniform-convergence}, the triangle inequality, the reverse triangle inequality, and $\vert a^2 - b^2 \vert = \vert a + b \vert \vert a - b \vert$ for any $a, b \in \mathbb{R}$, we obtain
\begin{align*}
\bigl\vert \widehat{G}_n(\sigma) - \widetilde{G}_n(\sigma) \bigr\vert &\leq \sum_{j = 1}^d \Biggl\vert \mathcal{W}_2\Bigl( \mathrm{std}\bigl( \frac{1}{n} \sum_{i = 1}^n \delta_{\widehat{R}_j^{(i)}(\sigma)} \bigr), \mathcal{N}(0, 1) \Bigr)^2 - \mathcal{W}_2\Bigl( \frac{1}{n} \sum_{i = 1}^n \delta_{\widetilde{R}_{j, A_j(\sigma)}^{(i)}}, \mathcal{N}(0, 1) \Bigr)^2 \Biggr\vert \\
&\leq 4 \sum_{j = 1}^d \mathcal{W}_2\Bigl( \mathrm{std}\bigl( \frac{1}{n} \sum_{i = 1}^n \delta_{\widehat{R}_j^{(i)}(\sigma)} \bigr), \frac{1}{n} \sum_{i = 1}^n \delta_{\widetilde{R}_{j, A_j(\sigma)}^{(i)}} \Bigr) \\
&\leq 40d \delta_n.
\end{align*}

Therefore, on the same event, we simultaneously have
\[
\max_{\sigma \in \mathfrak{S}_d} \bigl\vert \widehat{G}_n(\sigma) - \widetilde{G}_n(\sigma) \bigr\vert \leq 40d \delta_n.
\]

Both $\widehat{G}_n(\sigma)$ and $G(\sigma)$ belong to $[0, 2d]$ because every distribution entering these objectives has unit second moment. Applied to $\Sigma^{-1/2}X$, the sub-Gaussian covariance concentration theorem \citep[Theorem~4.7.1 and Remark~4.7.3]{vershynin2026high} gives a universal constant $C_0 \geq 1$ such that, for every $u \geq 0$,
\[
\mathbb{E}\bigl[ \delta_n \bigr] \leq C_0 K^2\biggl( \sqrt{\frac{d}{n}} + \frac{d}{n} \biggr), \quad \mathbb{P}\Biggl( \delta_n > C_0 K^2\biggl( \sqrt{\frac{d + u}{n}} + \frac{d + u}{n} \biggr) \Biggr) \leq 2 \exp(-u).
\]

Taking $u = \frac{n}{64 C_0^2 K^4}$, the condition $n \geq 64 C_0^2 K^4 d$ implies $d + u \leq 2u$, hence
\[
C_0 K^2\biggl( \sqrt{\frac{d + u}{n}} + \frac{d + u}{n} \biggr) \leq \frac{1}{\sqrt{32}} + \frac{1}{32 C_0 K^2} \leq \frac{1}{4} \implies \mathbb{P}\bigl( \delta_n > \frac{1}{4} \bigr) \leq 2 \exp\biggl( -\frac{n}{64 C_0^2 K^4} \biggr).
\]

Since $\exp(-u) \leq u^{-1/2}$ and $n \geq d$, decomposing according to the event $\delta_n \leq 1/4$ therefore gives
\[
\mathbb{E}\Bigl[ \max_{\sigma \in \mathfrak{S}_d} \bigl\vert \widehat{G}_n(\sigma) - G(\sigma) \bigr\vert \Bigr] \leq C_p(P) \Bigl( d^{5/4} n^{-1/4} \log(2n + 1)^{1/p + 1/4} + K^2 d^{3/2} n^{-1/2} \Bigr).
\]

The condition $n \geq K^8 d$ makes the second term no larger than the first, after enlarging $C_p(P)$ by a factor of two, which proves the stated expectation bound.

Finally, by Assumptions~\ref{ass:causal-noise-normalization} and~\ref{ass:causal-non-gaussianity} and Theorem~\ref{thm:causal-order-identification}, $\Gamma > 0$ when $\mathcal{I}^\star \neq \mathfrak{S}_d$, and the proof of Corollary~\ref{cor:dag-order-consistency} shows that
\[
\mathbb{P}\bigl( \hat{\sigma}_n \notin \mathcal{I}^\star \bigr) \leq \frac{2}{\Gamma}\mathbb{E}\Bigl[ \max_{\sigma \in \mathfrak{S}_d} \bigl\vert \widehat{G}_n(\sigma) - G(\sigma) \bigr\vert \Bigr] \leq \frac{C_p(P) d^{5/4} n^{-1/4} \log(2n + 1)^{1/p + 1/4}}{\Gamma}.
\]

If $\mathcal{I}^\star = \mathfrak{S}_d$, the error probability is zero by definition.
\end{proof}

\subsection{Proof of Theorem~\ref{thm:greedy-equal-score}} \label{pf:greedy-equal-score}

\begin{proof}
At step $t \geq 1$, let $A_t \subseteq \llbracket d \rrbracket$ be the selected variables and $U_t \coloneqq \llbracket d \rrbracket \setminus A_t$. Set
\[
\forall j \in U_t, \, R_j(t) \coloneqq X_j - \mathrm{proj}_{\mathrm{span} \bigl\{ X_k : k \in A_t \bigr\}} X_j,
\]
and in particular
\[
\forall j \in \llbracket d \rrbracket, \, R_j(1) = X_j.
\]

By Equation~\eqref{eq:greedy-equal-score}, every exogenous variable has score $c$. By Assumptions~\ref{ass:causal-noise-normalization} and~\ref{ass:causal-non-gaussianity}, Lemma~\ref{lem:strict-wasserstein} gives every nontrivial mixture a score strictly below $c$. Hence, at the first step, the greedy rule selects a structural noise, and thus, by Lemma~\ref{lem:prop-parents}, an exogenous variable.

By an induction argument similar to that of Proof~\ref{pf:permutation-causal}, and by Assumptions~\ref{ass:causal-noise-normalization} and~\ref{ass:causal-non-gaussianity}, after residualizing the variables in $U_t$ on $A_t$ at step $t \geq 1$, the current variables $\bigl( R_j(t) : j \in U_t \bigr)$ form a linear non-Gaussian acyclic model on $U_t$. By Equation~\eqref{eq:greedy-equal-score}, the reduced model retains the common Wasserstein non-Gaussianity score $c$, so the same argument shows that the greedy step again selects an exogenous variable. The induction continues until all variables have been selected, and the returned order belongs to $\mathcal{I}^\star$.
\end{proof}

\section{Reproducibility} \label{sec:reproducibility}

All experiments were run exclusively on an AMD Ryzen 5 5600 CPU under Windows 11. The accompanying GitHub repositories include pinned package versions to ensure full reproducibility.

\section*{LLM Usage}

The authors used Large Language Models (LLMs) for editorial assistance, including improving grammar and concision and suggesting alternative phrasings. The authors take full responsibility for the scientific ideas, mathematical statements, proofs, experimental design, implementations, and final manuscript content.

\end{document}